\documentclass[11pt,a4paper]{amsart}
\usepackage{comment,listings,framed,graphicx,url,xcolor,enumerate, hyperref, amsmath, float, multicol, multirow}
\usepackage{fullpage} 
\usepackage{todonotes}
\definecolor{mild}{rgb}{1,0.98,0.8}
\usepackage[utf8]{inputenc}    
\usepackage[T1]{fontenc}       
\usepackage{lmodern}           
\usepackage[margin=.9in]{geometry}
\usepackage{graphicx}          
\usepackage{amsmath,amssymb}   
\usepackage{booktabs}          
\usepackage{caption}           
\usepackage{natbib}
\usepackage{subcaption}        
\usepackage{microtype}         
\usepackage{hyperref}          
\usepackage{natbib}            
\usepackage{xcolor}            
\usepackage{amsthm}
\usepackage{amsmath}

\allowdisplaybreaks[4]
\usepackage{mathtools}
\mathtoolsset{showonlyrefs}

\numberwithin{equation}{section}
\numberwithin{figure}{section}
\numberwithin{table}{section}

\newtheorem{theorem}{Theorem}[section]   
\newtheorem{lemma}[theorem]{Lemma}       
\newtheorem{corollary}[theorem]{Corollary} 
\theoremstyle{definition}
\newtheorem{definition}[theorem]{Definition}
\newtheorem{remark}[theorem]{Remark}

\hypersetup{
  colorlinks=true,
  linkcolor=blue,
  citecolor=black,
  urlcolor=blue
}

\DeclareMathOperator{\re}{Re}
\DeclareMathOperator{\im}{Im}

\title{Geometry-Conditioned Fourier Neural Operators for Cubic Nonlinear Schr\"odinger Dynamics on Periodic Domains}
\author[Oguadimma, Obieke, and Yu]{Emmanuel E. Oguadimma,\,\, Victory C. Obieke,\,\, and \, Xueying Yu}

\address{Emmanuel E. Oguadimma
\newline \indent Department of Mathematics, Oregon State University\indent 
\newline \indent  Kidder Hall 368
Corvallis, OR 97331 \indent 
}
\email{oguadime@oregonstate.edu}

\address{Victory C. Obieke
\newline \indent Department of Mathematics, Oregon State University\indent 
\newline \indent  Kidder Hall 368
Corvallis, OR 97331 \indent 
}
\email{obiekev@oregonstate.edu}

\address{Xueying  Yu
\newline \indent Department of Mathematics, Oregon State University\indent 
\newline \indent  Kidder Hall 368
Corvallis, OR 97331 \indent 
}
\email{xueying.yu@oregonstate.edu}

\begin{document}


\begin{abstract}
We consider the cubic nonlinear Schr\"odinger (NLS) equation on two-dimensional flat tori with varying aspect ratios. In this formulation, the choice of aspect ratio governs the Fourier resonance structure, so rational and irrational geometries can exhibit different high-frequency cascade behaviors. We present a geometry-conditioned Fourier neural operator (GC-FNO) for the cubic defocusing NLS equation, where the input consists of the real and imaginary parts of the solution together with the aspect-ratio parameter \(\omega^2\). The model is trained to approximate the one-step solution operator and is evaluated on unseen trajectories generated from random-phase initial data using Fourier pseudospectral method. Our numerical experiments show that the learned operator captures the main solution dynamics on both tori and reproduces the distinct Sobolev norm behavior of the two geometries, with stronger \(H^2\)-growth on the rational torus and more constrained behavior on the irrational torus, consistent with the findings of \cite{hrabski2021energy}. We perform ablation studies to examine the roles of retained Fourier modes, activation functions, Fourier-layer depth, and explicit geometry conditioning. The results indicate that including $\omega^2$ improves long-time predictive accuracy, especially for the rational geometry, and supports the use of geometry-aware neural operators for learning spectral-transfer phenomena in nonlinear dispersive partial differential equations.

\end{abstract}

\maketitle

\setcounter{tocdepth}{1}
\tableofcontents

\parindent = 10pt     
\parskip = 8pt

\section{Introduction}\label{sec:intro}

The nonlinear Schr\"odinger (NLS) equation is a central dispersive partial differential equation (PDE) in mathematical physics. It arises in nonlinear optics and plasma physics, including the self-focusing and collapse of intense beams in Kerr media \citep{NewellMoloney1992,SulemSulem1999}, in the modulation of deep-water waves \citep{Zakharov1968}, and in the Gross--Pitaevskii description of Bose--Einstein condensates at ultra-low temperatures \citep{bao2007nonlinear}. Beyond its role as a canonical nonlinear wave equation, the cubic NLS also serves as a fundamental model for studying long-time energy redistribution across Fourier modes in various geometries. In recent years, this problem has been investigated under the broader perspective of weak turbulence theory, where one seeks to understand how energy migrates from low to high frequencies \cite{Bourgain1996,hrabski2021energy}. In this formulation, the geometry of the underlying domain directly shapes the resonance structure and therefore the efficiency of the cascade.

Recent developments on energy transfer for the cubic defocusing NLS on irrational tori \cite{hrabski2021energy}, have analytically shown that irrational aspect ratios substantially weaken the transfer of energy to high frequencies compared with the rational case. The analysis proves improved upper bounds on Sobolev norm growth for generic irrational tori, establishes barriers to the propagation of Fourier support for initially frequency-localized data, and studies the associated quasi-resonant system. Just as importantly for applications, the numerical study demonstrates that the energy cascade is consistently slower on irrational tori than on rational tori, and explains this difference through the geometry of quasi-resonant sets. This has a concrete implication beyond pure analysis. It places a serious caveat on periodic-domain simulations used in wave turbulence, because two simulations with different domain aspect ratios may not represent the same small-scale cascade physics, even when the resolution is very high. In other words, the interaction between dispersion relation and domain geometry must be accounted for when periodic-box computations are used as surrogates for homogeneous turbulence on large or effectively unbounded domains.

Classical numerical methods are standard approaches for solving the NLS equation and its variants. Split-step Fourier, time-splitting spectral, and pseudo-spectral techniques can provide high accuracy and are especially effective for dispersive wave propagation \cite{weideman1986splitstep,bao2003timesplitting}. However, repeated simulations over many initial conditions, parameter values, or geometries can become computationally expensive, and their accuracy can depend on the structure of the time-integration method and its treatment of symmetries and invariants \cite{oguadimma2026foundational}. This limitation has motivated growing interest in scientific machine learning methods that learn the behavior of the solution directly from data. In general, one can separate deep-learning approaches to PDEs into two classes. One class includes instance-based solvers such as physics-informed neural networks (PINNs) \cite{Raissi2019PINNs} and related methods, including the Deep Ritz method \cite{Yu2018DeepRitz}, the Deep Galerkin Method \cite{SirignanoSpiliopoulos2018}, weak adversarial networks \cite{ZangBaoYeZhou2020},  and domain-decomposition extensions such as cPINNs \cite{JagtapKharazmiKarniadakis2020} and XPINNs \cite{JagtapKarniadakis2020}, which are often accurate but typically require retraining when conditions such as parameters change. Structure-preserving learning approaches have also shown that the incorporation of geometric and physical constraints can improve long-time fidelity and reduce drift in learned dynamical systems \cite{zhong2020symoden,obieke2025structure, jin2020sympnets, cranmer2020lagrangian, obieke2026structure, greydanus2019hamiltonian, hernandez2021structure}. The other class includes operator-learning approaches, which aim to learn mappings between infinite-dimensional function spaces and therefore can generalize over families of PDE instances \cite{Lu2021DeepONet,Li2021FNO,Kovachki2021FNO}. Among these approaches, the Fourier neural operator (FNO) has emerged as a particularly influential architecture because it parameterizes integral kernels in Fourier space, efficiently captures multiscale structure, and is well suited for PDEs whose dynamics are naturally organized by frequency interactions \cite{Li2021FNO,Kovachki2021FNO}. Its frequency-space formulation has supported applications in fluid--structure interaction \cite{xiao2024fourier}, multiphase flow \cite{wen2022u}, and heterogeneous material modeling \cite{You2022IFNO}.

A burgeoning body of work has shown that the FNO algorithm can effectively capture the dynamics of nonlinear waves governed by Schr\"odinger equations. In particular, it has been used to investigate parametric soliton-state transitions in the NLS, Hirota, and $\mathcal{PT}$-symmetric NLS equations, showing that a single trained network can learn families of complex nonlinear waves across parameter ranges \cite{ZhongYanTian2023}. It has also been extended to coupled NLS systems, where the interaction between multiple components, multiple physical effects, and parameter-dependent localized waves makes the dynamics even more intricate \cite{RenTian2026}. The present work is based on two observations. On the one hand, the irrational-torus NLS problem is scientifically rich, since recent mathematical analysis have revealed clear and practically significant differences between rational and irrational geometries \cite{hrabski2021energy, deng2019growth, planchon2017growth}. On the other hand, the literature suggest that operator-learning architectures are especially effective when the essential dynamics are encoded in frequency-space interactions and one seeks to learn an entire family of solutions \cite{Li2021FNO,ZhongYanTian2023,RenTian2026}. 

In this paper, we propose a geometry-conditioned Fourier neural operator (GC-FNO) formulation for the two-dimensional cubic defocusing NLS on rational and irrational tori. In contrast to previous studies that focus primarily on parametric families of explicit soliton or rogue-wave solutions \cite{ZhongYanTian2023, RenTian2026}, our focus is on how the torus aspect ratio changes the resonance structure of the Fourier lattice and thereby alters the observed energy-transfer dynamics. Our formulation follows the time-marching perspective of prior FNO work \cite{Li2021FNO}, but incorporates the aspect-ratio parameter of the torus as an additional input channel so that a single learned operator can distinguish between the rational and irrational geometries. The model learns a one-step solution operator from the current complex-valued state to the next state, and long-time predictions are obtained autoregressively by feeding each prediction back into the network. This setup lends itself naturally to the present problem, as it enables the learned operator to propagate geometry-dependent dynamics over long time intervals. Our numerical experiments follow the scientific framework of \cite{hrabski2021energy}. We consider random-phase initial data with compact Fourier support and compare the dynamics on the rational torus and the irrational torus, and generate reference solutions using the Fourier pseudo-spectral method with fourth-order Runge-Kutta time stepping.

What remains of the paper is organized as follows. Section~\ref{sec: theory} presents the theoretical background and introduces the NLS model on rational and irrational tori. Section~\ref{sec: oplearn} describes the operator-learning framework considered in this work, which includes the Fourier Neural Operator. Section~\ref{sec: experiment} presents numerical experiments, including data generation, computational setup, training, evaluation and ablation study.

\subsection*{Acknowledgment}
We would like to thank Alexander Hrabski for helpful discussions. X.Y. is partially supported by NSF DMS-2306429.

\section{Preliminaries}\label{sec: theory}
In this section, we introduce the two-dimensional cubic nonlinear Schr\"odinger equation on flat tori, define the rational and irrational geometries considered here, and establish the notation used throughout. We also briefly recall the analytical results most relevant to our numerical study. 
For further details, including rigorous statements and proofs, we refer the reader to \cite{hrabski2021energy,staffilani2020stability,colliander2010transfer, giuliani2022sobolev}.

\subsection{Fourier Transforms on Tori}

In \cite{hrabski2021energy}, the authors consider the scaled torus
\begin{align}\label{def:scaled-torus}
\mathbb{T}^2_{\underline{\omega}} =\mathbb{R}/\omega_1\mathbb{Z}\times \mathbb{R}/\omega_2\mathbb{Z}, \quad \underline{\omega}=(\omega_1,\omega_2)\in\mathbb{R}_+^2,
\end{align}
where $\omega_1$ and $\omega_2$ denote the periods in the two spatial directions and hence determine the geometry. The domain is called \emph{rational} if $\omega_1^2/\omega_2^2 \in\mathbb{Q}$ and \emph{irrational} otherwise; the two cases correspond to commensurate and incommensurate squared side lengths. Without loss of generality one reduces to $\underline{\omega}=(1,\omega)$, so that
\begin{equation}\label{def:torus}
\mathbb{T}^2_{\underline{\omega}}:=\mathbb{R}/\mathbb{Z}\times\mathbb{R}/\omega\mathbb{Z},
\end{equation}
and $\mathbb{T}^2_{\underline{\omega}}$ is \emph{irrational} if $\omega^2\notin\mathbb{Q}$ and \emph{rational} otherwise.

To best accommodate our simulations, we work with an equivalent $2\pi$-periodic formulation in which the geometry enters through an anisotropic Laplacian rather than through the domain. We pose the problem on the square torus
\begin{align}\label{def:square-torus}
\mathbb{T}^2=\mathbb{R}/2\pi\mathbb{Z}\times\mathbb{R}/2\pi\mathbb{Z},
\end{align}
and replace the standard Laplacian by the anisotropic operator
\begin{align}\label{def:aniso-laplacian}
\Delta_{\omega}:=\partial_x^2+\omega^2\partial_y^2,
\end{align}
writing $\Delta$ for $\Delta_\omega$ when the context is clear.

For every $f\in L^2(\mathbb{T}^2)$, we write its Fourier expansion
\begin{equation}\label{def:fourier-series}
    f(x,y)
    =
    \sum_{k=(m,\ell)\in\mathbb{Z}^2}
    \widehat{f}_k\,
    e^{ i(mx + \ell y)},
\end{equation}
with Fourier coefficients
\begin{equation}\label{def:fourier-coeff}
    \widehat{f}_k
    :=
    \frac{1}{(2\pi)^2}
    \int_{\mathbb{T}^2}
    f(x,y)\,e^{-i(mx+\ell y)}\,dx\,dy.
    \qquad
    k=(m,\ell)\in\mathbb{Z}^2.
\end{equation}
Parseval's identity reads
\begin{equation}\label{parseval}
\|f\|_{L^2(\mathbb{T}^2)}^2=(2\pi)^2\sum_{k\in\mathbb{Z}^2}|\widehat{f}_k|^2.
\end{equation}

With this basis the anisotropic Laplacian acts diagonally,
\begin{equation}\label{def:laplacian-eigenvalue}
\widehat{(\Delta_\omega f)}_k=-\lambda_k\,\widehat{f}_k,
\qquad
\lambda_k:=m^2+\omega^2\ell^2,\quad k=(m,\ell)\in\mathbb{Z}^2,
\end{equation}
so that all geometric information is carried by the eigenvalues $\lambda_k$. This formulation is equivalent to \eqref{def:torus} after rescaling the $y$-variable, and it has the advantage that the rational and irrational cases share the same $2\pi$-periodic grid, differing only through the value of $\omega$ in $\lambda_k$.

\subsection{Function Spaces}

\begin{definition}\label{def:sobolevH}
For $s \in [0,\infty)$, define the Sobolev space $H^s(\mathbb{T}^2)$ as the closure 
of smooth functions $f:\mathbb{T}^2\to\mathbb{C}$ with Fourier expansion
\begin{align}
f(x,y)=\sum_{k=(m,\ell)\in\mathbb{Z}^2}\widehat{f}_k\,e^{i(mx+\ell y)}
\end{align}
under the norm
\begin{align}
    \|f\|_{H^s(\mathbb{T}^2)} 
    := 
    \left(\sum_{k=(m,\ell)\in\mathbb{Z}^2} \langle k\rangle^{2s}|\widehat{f}_k|^2
    \right)^{1/2},
\end{align}
where $\langle k\rangle := (1 + |m|^2 + |\ell|^2)^{1/2}$. 
\end{definition}

\begin{definition}\label{def:sobolev}
For $x=\{x_n\}_{n\in\mathbb{Z}^2}$ and $s\in[0,\infty)$, define the
\emph{Sobolev norm}
\begin{align}
    \|x\|_{h^s}
    :=
    \sqrt{\sum_{n\in\mathbb{Z}^2}|x_n|^2\langle n\rangle^{2s}},
    \qquad
    \langle n\rangle := \sqrt{|n|^2+1},
\end{align}
and the \emph{sequence Sobolev space}
\begin{align}
    h^s
    :=
    h^s(\mathbb{Z}^2)
    :=
    \bigl\{
        x = \{x_n\}_{n\in\mathbb{Z}^2}
        :
        \|x\|_s < \infty
    \bigr\}.
\end{align}
In particular, $h^0(\mathbb{Z}^2) = \ell^2(\mathbb{Z}^2)$.
\end{definition}

In fact, thanks to the Fourier expansion \eqref{def:fourier-series} and Parseval's identity \eqref{parseval}, the function space $H^s(\mathbb{T}^2)$ is identified with $h^s(\mathbb{Z}^2)$ through the isometric isomorphism $f \mapsto \{\widehat{f}_k\}_{k\in\mathbb{Z}^2}$, and we write $\|f\|_{H^s} := \|\{\widehat{f}_k\}\|_{h^s}$.

\subsection{Littlewood-Paley projections}
Next we define the Littlewood-Paley projectors. Fix an even, non-increasing 
function $\eta\in C_c^\infty(\mathbb{R};[0,1])$ with
\begin{align}
    \eta(t)=1 \ \text{ for } |t|\le 1,
    \qquad
    \eta(t)=0 \ \text{ for } |t|\ge 2.
\end{align}
For a dyadic integer $N\in 2^{\mathbb{N}_0}=\{1,2,4,8,\dots\}$, define the symbols 
$\eta_N:\mathbb{Z}^2\to[0,1]$ by
\begin{align}
    \eta_1(\xi):=\eta(|\xi|),
    \qquad
    \eta_N(\xi):=\eta\!\left(\frac{|\xi|}{N}\right)-\eta\!\left(\frac{2|\xi|}{N}\right)
    \ \ \text{for } N\ge 2.
\end{align}
By construction these form a partition of unity on $\mathbb{Z}^2$,
\begin{align}
    \sum_{N\in 2^{\mathbb{N}_0}}\eta_N(\xi)=1
    \qquad\text{for every }\xi\in\mathbb{Z}^2,
\end{align}
with $\eta_N$ supported in the annulus $\{N/2\le|\xi|\le 2N\}$ for $N\ge 2$ and in the 
ball $\{|\xi|\le 2\}$ for $N=1$.

The \emph{Littlewood--Paley projector} $P_N$ is the Fourier multiplier with symbol 
$\eta_N$, i.e.\ for $f$ with Fourier coefficients $\{\widehat{f}_k\}_{k\in\mathbb{Z}^2}$,
\begin{align}
    \widehat{(P_N f)}_k:=\eta_N(k)\,\widehat{f}_k,
    \qquad k\in\mathbb{Z}^2.
\end{align}
For any threshold $N\in(0,\infty)$ we define the low- and high-frequency projections
\begin{align}
    P_{\le N}:=\sum_{\substack{M\in 2^{\mathbb{N}_0}\\ M\le N}}P_M,
    \qquad
    P_{>N}:=\sum_{\substack{M\in 2^{\mathbb{N}_0}\\ M>N}}P_M
    =I-P_{\le N}.
\end{align}
Each $P_N$, $P_{\le N}$, and $P_{>N}$ is a bounded linear operator on 
$H^s(\mathbb{T}^2)$ for every $s\ge 0$, with operator norm at most $1$; in particular 
the projections do not increase the $H^s$ norm.

\subsection{Useful Estimates}
We record two properties of these spaces that are used repeatedly in the sequel. For $K>0$ let $P_{\le K}$ denote the Fourier projection onto frequencies 
$\Lambda_K:=\{k=(m,\ell)\in\mathbb{Z}^2:|m|,|\ell|\leq K\}$, that is, 
$\widehat{P_{\leq K}f}_k=\widehat{f}_k$ for $k\in\Lambda_K$ and $0$ otherwise.

\begin{lemma}[Bernstein inequality]\label{lem:bernstein}
Let $0\leq s\leq s'$. For every $f\in H^{s'}(\mathbb{T}^2)$ and every $K>0$,
\begin{align}\label{eq:bernstein-low}
\|P_{\leq K}f\|_{H^{s'}}\leq \langle K\rangle^{\,s'-s}\,\|P_{\le K}f\|_{H^s},
\end{align}
and dually, for every $f\in H^{s'}(\mathbb{T}^2)$ frequency-supported outside 
$\Lambda_K$ \textnormal{(}i.e.\ $\widehat{f}_k=0$ for $k\in\Lambda_K$\textnormal{)},
\begin{align}\label{eq:bernstein-high}
\|f\|_{H^s}\leq \langle K\rangle^{\,-(s'-s)}\,\|f\|_{H^{s'}}.
\end{align}
\end{lemma}

\begin{lemma}[Algebra property]\label{lem:algebra}
Let $s>1$. Then $H^s(\mathbb{T}^2)$ is a Banach algebra under pointwise 
multiplication, that is, there exists a constant $C_s>0$, depending only on $s$, such that 
for all $f,g\in H^s(\mathbb{T}^2)$,
\begin{align}\label{eq:algebra}
\|fg\|_{H^s}\leq C_s\,\|f\|_{H^s}\,\|g\|_{H^s}.
\end{align}
In particular, for any three functions $f,g,h\in H^s(\mathbb{T}^2)$ the cubic 
expression obeys
\begin{align}\label{eq:trilinear}
\|fgh\|_{H^s}\leq C_s\,\|f\|_{H^s}\,\|g\|_{H^s}\,\|h\|_{H^s}.
\end{align}
\end{lemma}

\subsection{The Model}
We consider the cubic defocusing nonlinear Schr\"odinger equation on $\mathbb{T}^2$,
\begin{equation}\label{eq:nls-ivp}
\begin{cases}
i\partial_t \psi + \Delta_\omega \psi = |\psi|^2\psi,
&  (t,x,y)\in\mathbb{R}\times\mathbb{T}^2,\\[2pt]
\psi(0,\cdot) = \psi_0 \in H^s(\mathbb{T}^2).
\end{cases}
\end{equation}
Writing the solution in terms of its Fourier coefficients,
\begin{align}\label{exactsol}
\psi(t,x,y)=\sum_{k=(m,\ell)\in\mathbb{Z}^2}\widehat{\psi}_k(t)\,e^{i(mx+\ell y)},
\end{align}
substituting into \eqref{eq:nls-ivp}, and using \eqref{def:laplacian-eigenvalue} together with
\begin{align}
\widehat{(|\psi|^2\psi)}_k
=\sum_{\substack{k_1-k_2+k_3=k \\ k_j\in\mathbb{Z}^2}}
\widehat{\psi}_{k_1}\,\overline{\widehat{\psi}_{k_2}}\,\widehat{\psi}_{k_3},
\end{align}
the equation \eqref{eq:nls-ivp} is equivalent to the infinite system of ODEs for the Fourier coefficients
\begin{equation}\label{nlsmode}
    i\partial_t\widehat{\psi}_k
    -
    \lambda_k\widehat{\psi}_k
    =
    \sum_{\substack{k_1-k_2+k_3=k\\k_j\in\mathbb{Z}^2}}
    \widehat{\psi}_{k_1}
    \overline{\widehat{\psi}_{k_2}}
    \widehat{\psi}_{k_3},
\end{equation}
where $\lambda_k = m^2+\omega^2\ell^2$ for $k=(m,\ell) \in \mathbb{Z}^2$.

The cubic nonlinearity in \eqref{nlsmode} couples quartets of Fourier modes $(k_1,k_2,k_3,k_4)$ satisfying the momentum relation $k_1+k_2=k_3+k_4$. To see why a distinguished subclass of these quartets governs the long-time dynamics, it is convenient to pass to the interaction representation by removing the linear flow. 
Writing $\widehat{\psi}_k(t)=e^{-i\lambda_k t}a_k(t)$, the linear part of \eqref{nlsmode} is absorbed into the exponential, and the equation for the slowly varying amplitudes $a_k$ becomes
\begin{equation}\label{eq:interaction}
i\partial_t a_k=\sum_{\substack{k_1-k_2+k_3=k\\k_j\in\mathbb{Z}^2}}e^{\,i\Omega(k_1,k_2,k_3,k)\,t}\,a_{k_1}\overline{a_{k_2}}a_{k_3},
\quad
\Omega:=\lambda_{k_1}-\lambda_{k_2}+\lambda_{k_3}-\lambda_k,
\end{equation}
where, relabeling $k_4=k$, the momentum constraint reads $k_1+k_3=k_2+k_4$ and the phase is the frequency mismatch
\begin{equation}\label{eq:mismatch}
\Omega(k_1,k_2,k_3,k_4)
:=\lambda_{k_1}+\lambda_{k_2}-\lambda_{k_3}-\lambda_{k_4}.
\end{equation}
Each term in \eqref{eq:interaction} thus carries an oscillatory factor $e^{i\Omega t}$. When $\Omega\neq 0$, this phase rotates in time, and over long intervals the contribution of that quartet averages to a negligible amount by the method of stationary phase. Such rapid oscillation causes successive contributions to cancel, so the interaction transfers essentially no energy between the modes. When $\Omega=0$, the phase is stationary, $e^{i\Omega t}\equiv 1$, and the corresponding term contributes coherently and cumulatively for all time. These persistent interactions are the ones that actually drive the transfer of energy across frequencies. Accordingly, a quartet $(k_1,k_2,k_3,k_4)$ with $k_1+k_2=k_3+k_4$ is called \emph{resonant} precisely when its frequency mismatch vanishes,
\begin{equation}\label{eq:resonance}
\lambda_{k_1}+\lambda_{k_2}=\lambda_{k_3}+\lambda_{k_4}
\quad\Longleftrightarrow\quad
\Omega(k_1,k_2,k_3,k_4)=0.
\end{equation}
Using $\lambda_k=m^2+\omega^2\ell^2$, the resonance condition \eqref{eq:resonance} separates into its $x$- and $y$-components,
\begin{equation}\label{eq:mismatch-split}
\Omega(k_1,k_2,k_3,k_4)
=\underbrace{(m_1^2+m_2^2-m_3^2-m_4^2)}_{=:p\,\in\,\mathbb{Z}}
+\omega^2\underbrace{(\ell_1^2+\ell_2^2-\ell_3^2-\ell_4^2)}_{=:q\,\in\,\mathbb{Z}}.
\end{equation}
When $\omega^2\notin\mathbb{Q}$, the condition $p+\omega^2q=0$ with $p,q\in\mathbb{Z}$ forces $p=q=0$ separately, decoupling the resonances into two independent one-dimensional conditions and greatly restricting the resonant set; when $\omega^2\in\mathbb{Q}$, genuinely two-dimensional resonant interactions occur. Thus the rational and irrational cases are distinguished entirely through the arithmetic of $\lambda_k$, and the size of $\Omega$ is governed by how well $\omega^2$ is approximated by rationals. Consequently, on irrational tori the resonant interactions that drive efficient energy transfer to high frequencies are fewer, so both the transfer of energy and the growth of Sobolev norms are generally weaker than on rational tori \cite{hrabski2021energy}. If the aspect ratio remains close to rational, some of the growth mechanisms of the rational setting may persist. This sensitivity to the geometry of the domain is one of the main reasons rational and irrational tori provide a meaningful benchmark for the operator-learning study carried out in this paper.

\subsection{Known Theoretical Results}

In this subsection, we collect a few results that study the long-time behavior and growth of Sobolev norms of solutions on generic irrational tori and rational tori, and show that the growth on irrational tori is significantly more constrained than in the rational case.

On square tori,
\begin{theorem}[{\cite[Theorem 1.5]{HK26}}]\label{thm2.5}
For $N \in 2^{\mathbb{N}}$, denote by $\phi^N$ the function
\begin{align}
    \phi^N : = \mathcal{F}^{-1} (\chi_{N^{10} \mathbb{Z}^2} \cdot e^{-|\xi / N^{11}|^2})
\end{align}
Let $T>0$ and $\lambda >0$ be a small number. Let $\psi^N \in C_{loc}^{\infty} (\mathbb{R} \times \mathbb{T}^2)$ be the solution to \eqref{eq:nls-ivp} with the initial datum $\psi_0^N : = \lambda N^{-1} \phi^N$. We have
\begin{align}
    \| \psi_0^N\|_{L^2} \sim \lambda , \quad \log \| \psi_0^N\|_{H^1} \leq C \log N
\end{align}
and
\begin{align}
    \limsup_{N \to \infty} \| \psi^N (t) - e^{-3i t \lambda^2 \ln N} e^{it \Delta }\psi_0^N\|_{C^0 L^2 \cap L_{t,x}^4 ([0, \frac{T}{\log N}) \times \mathbb{T}^2)} =0 .
\end{align}
\end{theorem}

On generic irrational tori,
\begin{theorem}[{\cite[Theorem 1.1]{hrabski2021energy}}]\label{thm1.1}
Assume that the solution of a periodic NLS equation as in \eqref{eq:nls-ivp}, not necessarily on an irrational torus, satisfies for $s\gg 1$ the asymptotic estimate $\|\psi(t)\|_{H^s}\leq C|t|^{s\alpha}R^{2\alpha+1}$ for $\alpha>0$ and $R=\|\psi(0)\|_{H^s}$. Then on a torus $T^2_{\underline{\omega}}$, where $\omega\in\mathbb{R}^2_+$ is a generic irrational ordered pair, if the initial data has bounded frequency support, we can improve the estimate above for $|t|\gg 1$ to
\begin{equation*}
\|\psi(t)\|_{H^s}\leq C|t|^{s\left(\frac{\alpha(s+2+2\tau)+1}{3s+2\tau-2}\right)}
R^{\frac{s(2\alpha+1)}{3s+2\tau-2}}
L^{-\frac{3s}{3s+2\tau-2}},
\end{equation*}
where $\tau>1$ and the constant $C$ depends on $\tau$, $\omega$ and possibly on the size of the support, and $L$ is the $L^2$ norm of the initial data.
\end{theorem}

\begin{corollary}[{\cite[Corollary 1.2]{hrabski2021energy}}]\label{cor1.2}
Let $\varepsilon > 0$ and let $ \frac{5}{2} + (2\varepsilon)^{-1} < s$. Then the solution $\psi$ of the NLS equation in \eqref{eq:nls-ivp}, and with initial data of bounded frequency support, is such that
\begin{align}
    \|\psi(t)\|_{H^s} \le C |t|^{\varepsilon s} R^{2\varepsilon+1},
\end{align}
for $t$ large enough, where the constant $C$ only depends on $\omega$, the size of the support of the
initial data and its $L^2$ norm.
\end{corollary}

\begin{remark}[{\cite[Remark 1.3]{hrabski2021energy}}]
This corollary tells us that if the torus is irrational enough, any initial data
with bounded frequency support evolves into a solution $\psi$ such that for $s \ge 3$
\begin{align}
    \|\psi(t)\|_{H^s} \lesssim (1+|t|),
\end{align}
namely any Sobolev norm has at most a linear growth.
\end{remark}

The following result, used in \cite{hrabski2021energy}, distinguishes generic irrational aspect ratios from rational ones.

\begin{theorem}[Diophantine approximation]\label{diophantine}
There exists a full-measure subset $E\subset\mathbb{R}$ such that for every $\alpha\in E$ and all $C>0$, $\tau>0$, the inequality
\begin{equation}
\left|\frac{p}{q}-\alpha\right|\leq\frac{C}{q^{2+\tau}}
\end{equation}
has only finitely many solutions $(p,q)\in\mathbb{Z}\times\mathbb{Z}^+$.
\end{theorem}

This motivates the comparison between rational and irrational geometries in our experiments: following \cite{hrabski2021energy}, we use the rational torus $\omega^2=1$ as a benchmark and the irrational torus $\omega^2=\sqrt{2}$ as its 
irrational counterpart.

\section{Operator Learning}\label{sec: oplearn}

Let \(\Omega \subset \mathbb{R}^d\) be a bounded spatial domain. We denote by
\(\mathcal{A}=\mathcal{A}(\Omega;\, \mathbb{R}^{d_a})\) the space of admissible input functions  and by
\(\mathcal{U}=\mathcal{U}(\Omega;\, \mathbb{R}^{d_u})\) the space of output solution functions, taking values in $\mathbb{R}^{d_a}$ and $\mathbb{R}^{d_u}$ respectively. In this formulation, the object of interest is a map between function spaces, $\mathcal{J}^\dagger:\mathcal{A}\to \mathcal{U},$ which is typically nonlinear \cite{Kovachki2023NeuralOperator,Lu2021DeepONet}. Here \(a\in \mathcal{A}\) may represent an initial condition, coefficient field, forcing term, or other functional parameter, while \(u=\mathcal{J}^\dagger(a)\) denotes the corresponding solution.

Given observations $ \{(a_j,u_j)\}_{j=1}^N,$ with \(a_j\) sampled from a probability measure \(\nu\) on \(\mathcal{A}\), we aim to construct a parameterized approximation $\mathcal{J}_\theta:\mathcal{A}\to \mathcal{U},\,  \theta\in \Theta,$ for some finite-dimensional parameter space $\Theta$, such that \(\mathcal{J}_\theta(a)\approx \mathcal{J}^\dagger(a)\) for inputs
drawn from the same distribution. Formally, one seeks parameters minimizing an
expected discrepancy,
\[
    \min_{\theta\in\Theta}
    \mathbb{E}_{a\sim \nu}
    \left[
        \mathcal{C}\bigl(\mathcal{J}_\theta(a),\mathcal{J}^\dagger(a)\bigr)
    \right],
\]
where \(\mathcal{C}:\mathcal{U}\times \mathcal{U}\to \mathbb{R}\) is a chosen
cost functional. In practice, this expected loss is replaced by an empirical
loss over the available training data, for example
\[
    \min_{\theta\in\Theta}
    \frac{1}{N}\sum_{j=1}^N
    \mathcal{C}\bigl(\mathcal{J}_\theta(a_j),u_j\bigr).
\]
Although the mathematical formulation is posed over function spaces, the training data are observed on finite discretizations. If $\Omega_j=\{x_1,\dots,x_n\}\subset \Omega$ is a set of sampling points, then one has access to arrays $a_j|_{\Omega_j}\in \mathbb{R}^{n\times d_a}, \,  u_j|_{\Omega_j}\in \mathbb{R}^{n\times d_u}.$ The purpose of a neural operator is to learn a representation of \(\mathcal{J}^\dagger\) that is not tied to one particular discretization. Thus, the same learned parameters should define a mapping between functions, while their finite-dimensional realization may be evaluated on different grids, provided the relevant features of the functions are resolved \cite{Li2020GraphKernel,Li2021FNO,Kovachki2023NeuralOperator}.

\subsection{Fourier Neural Operator}

The Fourier neural operator constructs \(\mathcal{J}_\theta\) by combining local
pointwise transformations with nonlocal transformations in Fourier space. The input function is first lifted to a higher-dimensional channel space through a shallow fully connected neural network $ P:\mathbb{R}^{d_a}\to \mathbb{R}^{d_v}$, formally written as $v_0(x)=P(a(x))$. Here \(d_v\) is the width of the hidden representation. The lifted function is
then passed through a sequence of operator layers $v_0 \mapsto v_1 \mapsto \cdots \mapsto v_T,$ and the final representation is projected back to the target dimension by $u(x)=Q(v_T(x)), \, Q:\mathbb{R}^{d_v}\to \mathbb{R}^{d_u}.$ Equivalently, the full architecture may be written as
\begin{align}
    \mathcal{J}_\theta
    =
    Q\circ \mathcal{L}_T\circ \cdots \circ \mathcal{L}_1\circ P.
\end{align}
where \(\mathcal{L}_1,\ldots,\mathcal{L}_T\) denote the Fourier layers. Each layer updates the hidden function by
\begin{align}\label{iter}
    v_{t+1}(x)
    =
    \sigma
    \left(
        W v_t(x)
        +
        (\mathcal{K}(a;\,\phi)v_t)(x)
    \right),
    \quad
    t=0,\dots,T-1,
\end{align}
where \(\sigma: \mathbb{R}\to \mathbb{R}\) is a nonlinear activation function, \(W\) is the linear transformation, and
\(\mathcal{K}(a;\,\phi)\) is a kernel integral operator that is parameterized by $\phi$, written as
\begin{align}\label{kintop}
      (\mathcal{K}(a;\,\phi)v_t)(x)
    =
    \int_\Omega
    \kappa(x,y,a(x),a(y);\,\phi)v_t(y)\,dy.
\end{align}
Here, \(\kappa(\cdot;\,\phi)\) is a learned kernel function. If the kernel depends only on the relative displacement, so that $\kappa(x,y,a(x),a(y);\,\phi)=\kappa(x-y;\,\phi)$, then the kernel integral operator \eqref{kintop} becomes a convolution. By the convolution theorem,
\begin{align*}
    (\mathcal{K}(\phi)v_t)(x)
    =
    \mathcal{F}^{-1}
    \left(
        \mathcal{F}(\kappa_\phi)\,\cdot\mathcal{F}(v_t)
    \right)(x),
\end{align*}
or equivalently written as 
\begin{align}
    (\mathcal{K}(\phi)v_t)(x)
    =
    \mathcal{F}^{-1}
    \left(
        R_\phi \cdot \mathcal{F}(v_t)
    \right)(x),
\end{align}
where \(R_\phi\) represents the Fourier transform of the periodic function $\kappa$, which can be taken via the scalar convolution kernel. Thus, equation \eqref{iter} becomes,
\[
    v_{t+1}(x)
    =
    \sigma
    \left(
        W v_t(x)
        +
        \mathcal{F}^{-1}
        \left(
            R_\phi\cdot \mathcal{F}(v_t)
        \right)(x)
    \right),
\]
and the final output is obtained by applying the shallow fully connected layer $Q(\cdot)$ locally to $v_T$, giving 
$ u(x) = Q(v_T(x)).$ 

In the FNO architecture, the lifting layer \(P\) embeds the input field into a higher-dimensional latent representation, while the projection layer \(Q\) maps this representation back to the desired output dimension. Given an input function \(a(x)\in \mathbb{R}^{B\times H\times W\times C_1}\), the network learns an operator \(\mathcal J_{\theta}: \mathcal{A}\longmapsto \mathcal{U}\) such that \(u(x)\in \mathbb{R}^{B\times H\times W\times C_2}\), where \(B\) is the batch size, \(H\) and \(W\) are the spatial grid dimensions (in particular, the height and width respectively), and \(C_1\) and \(C_2\) denote the input and output channel dimensions. The main expressive component of the FNO is the Fourier layer, which updates the latent representation in spectral space and allows the model to capture global spatial interactions and complex multiscale features efficiently.

\begin{figure}[H]
    \centering
    \includegraphics[width=0.7\linewidth]{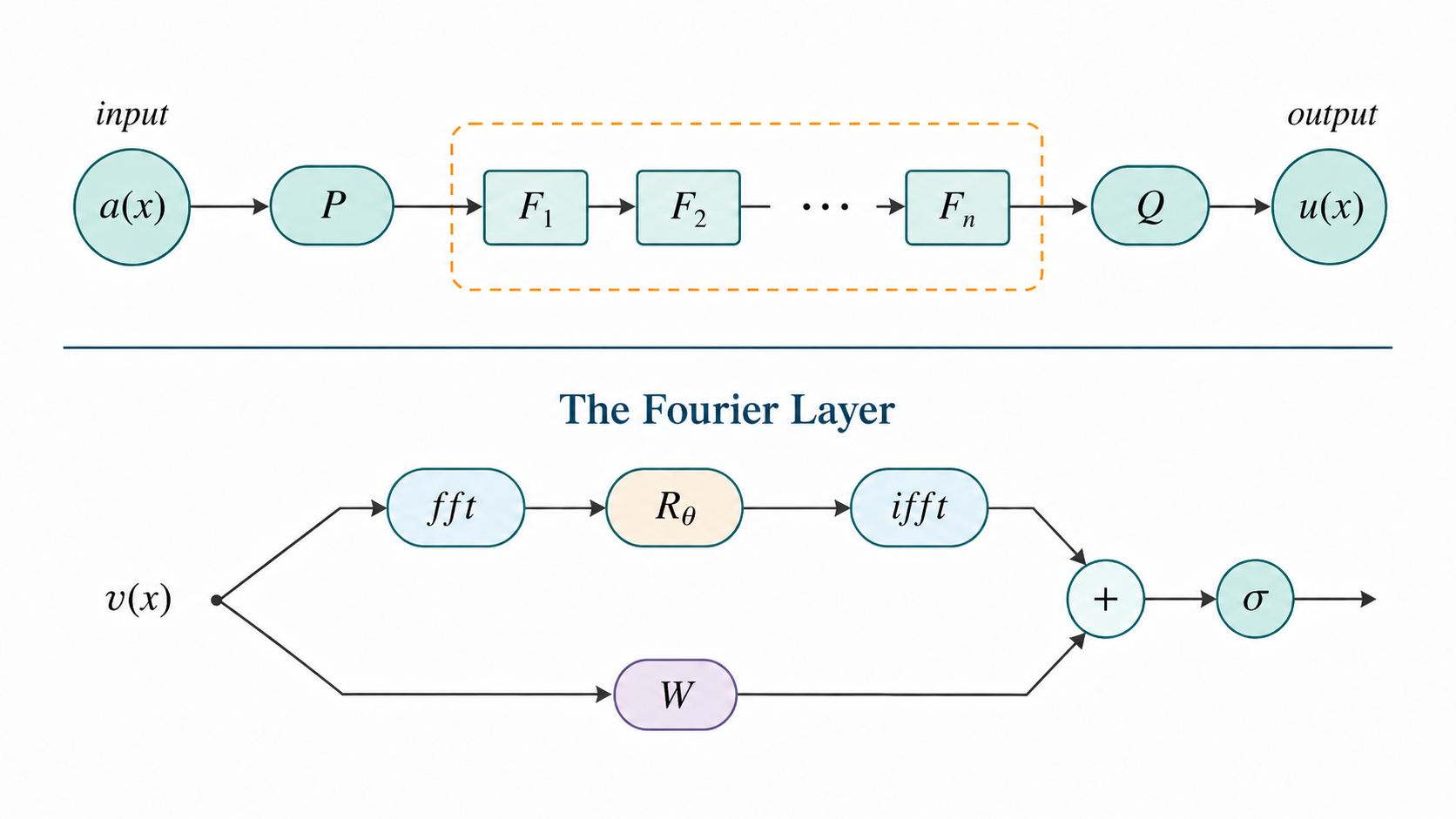}
    \caption{Schematic of the FNO architecture. The input function $a(x)$ is lifted by the shallow fully connected network $P$, propagated through a sequence of Fourier layers $F_1,\ldots,F_n$, and projected by $Q$ to produce the output $u(x)$. The bottom image shows the structure of a Fourier layer, where $v(x)$ is updated by combining the spectral convolution $\mathcal{F}^{-1}(R_\theta \cdot \mathcal{F}v)$ with the pointwise linear map $Wv$, followed by the activation function $\sigma$.}
\label{fig:fno_architecture}
\end{figure}
We now state an a posteriori error estimate for the learned solution. The estimate is formulated at the level of the Fourier coefficients used in any given reference solver. Thus, the finite set $\Lambda_k$ specifies the retained Fourier modes, and the nonlinear interaction is restricted to those modes. 

The residual \eqref{fnores} measures how well the learned Fourier coefficients satisfy the same truncated modal system as the reference solution. A small residual,
together with a small initial mismatch, implies that the learned solution remains close to the reference solution on the time interval \([0,T]\).

\begin{theorem}[Error estimate]\label{thm3.1}
Suppose $\omega,T>0$, $s>1$, and let 
$\Lambda_K=\{k=(m,\ell)\in\mathbb{Z}^2:\ |m|,|\ell|\le K\}$. For 
$k=(m,\ell)\in\mathbb{Z}^2$, define $\lambda_k=m^2+\omega^2\ell^2$. Let $\psi$ be a 
solution to \eqref{eq:nls-ivp} on $\mathbb{T}^2$ with Fourier coefficients 
$\widehat{\psi}_k$ satisfying the full mode equation
\begin{align}\label{nlsmode}
    i\partial_t\widehat{\psi}_k-\lambda_k\widehat{\psi}_k
    =\sum_{\substack{k_1-k_2+k_3=k\\ k_j\in\mathbb{Z}^2}}
    \widehat{\psi}_{k_1}\overline{\widehat{\psi}_{k_2}}\widehat{\psi}_{k_3},
    \qquad k\in\mathbb{Z}^2.
\end{align}
Let $\psi^\theta$ be the learned solution to the frequency-restricted equation
\begin{align}\label{eq:filtered-nls}
\begin{cases}
    i\partial_t \psi^\theta + \Delta_\omega \psi^\theta 
    = P_{\le K}\bigl(|\psi^\theta|^2\psi^\theta\bigr),
    & (t,x,y)\in[0,T]\times\mathbb{T}^2,\\[2pt]
    \psi^\theta(0,\cdot)=\psi^\theta_0\in P_{\le K}H^s(\mathbb{T}^2),
\end{cases}
\end{align}
with Fourier coefficients $\widehat{\psi}^{\,\theta}_k$, $k\in\Lambda_K$, and define 
the residual $r_k^\theta(t)$ by
\begin{align}\label{fnores}
    r_k^\theta(t)
    =
    i\partial_t\widehat{\psi}^{\,\theta}_k(t)
    -\lambda_k\widehat{\psi}^{\,\theta}_k(t)
    -\sum_{\substack{k_1-k_2+k_3=k\\ k_j\in\Lambda_K}}
    \widehat{\psi}^{\,\theta}_{k_1}\overline{\widehat{\psi}^{\,\theta}_{k_2}}\widehat{\psi}^{\,\theta}_{k_3},
    \qquad k\in\Lambda_K.
\end{align}
Assume the uniform bound
\begin{align}\label{mequ}
    \sup_{0\le t\le T}
    \Bigl(\|\psi(t)\|_{H^{s+1}}+\|\psi^\theta(t)\|_{H^s}\Bigr)\le M.
\end{align}
Then there exists a constant $C_s>0$, depending only on $s$, such that
\begin{align}\label{eq:err}
    \sup_{0\le t\le T}\|\psi(t)-\psi^\theta(t)\|_{H^s}
    \le
    \exp\bigl(C_sM^2T\bigr)
    \left[
        \|\psi(0)-\psi^\theta(0)\|_{H^s}
        +\int_0^T\|r^\theta(\tau)\|_{h^s}\,d\tau
        +\frac{C_sM^3T}{K}
    \right].
\end{align}
\end{theorem}
\begin{proof}
Define the truncated cubic convolution operator on the Fourier side,
\begin{align}
    \mathcal{N}(z)_k
    =
    \sum_{\substack{k_1-k_2+k_3=k\\ k_j\in\Lambda_K}}
    z_{k_1}\overline{z_{k_2}}z_{k_3},
    \qquad k\in\Lambda_K,
\end{align}
which restricts \emph{all} frequencies $k,k_1,k_2,k_3$ to $\Lambda_K$. For a sequence 
$z$ with frequencies restricted to $\Lambda_K$, the input restriction is automatic, 
and $\mathcal{N}(z)$ coincides with the projection onto $\Lambda_K$ of the cubic 
coefficient sequence of $u$, where $\widehat{u}=z$; equivalently
\begin{align}\label{eq:N-identity}
    \mathcal{N}(z)=P_{\le K}\bigl(|P_{\le K}u|^2\,P_{\le K}u\bigr).
\end{align}
In particular, $\mathcal{N}(\widehat{\psi}^{\,\theta})=P_{\le K}(|\psi^\theta|^2\psi^\theta)$ 
since $\psi^\theta$ has frequencies restricted to $\Lambda_K$, whereas for the full 
solution $\psi$ \eqref{eq:N-identity} gives 
$\mathcal{N}(\widehat{\psi})=P_{\le K}(|P_{\le K}\psi|^2 P_{\le K}\psi)$, which uses 
only the low modes of $\psi$.

\medskip
\noindent\textbf{Error equation.}
Let $\widehat{e}_k(t)=\widehat{\psi}_k(t)-\widehat{\psi}^{\,\theta}_k(t)$ for 
$k\in\Lambda_K$. The exact coefficients satisfy the full mode equation 
\eqref{nlsmode} with $k_j\in\mathbb{Z}^2$; separating the low-input triples from those 
reaching outside $\Lambda_K$, for $k\in\Lambda_K$,
\begin{align}\label{eq:psi-split}
    i\partial_t\widehat{\psi}_k-\lambda_k\widehat{\psi}_k
    =\mathcal{N}(\widehat{\psi})_k+a_k,
    \qquad
    a_k:=\sum_{\substack{k_1-k_2+k_3=k\\ \text{some }k_j\notin\Lambda_K}}
    \widehat{\psi}_{k_1}\overline{\widehat{\psi}_{k_2}}\widehat{\psi}_{k_3},
\end{align}
where $a=\{a_k\}_{k\in\Lambda_K}$ is the truncation term. The learned coefficients 
satisfy \eqref{fnores}, that is 
\begin{align}
    i\partial_t\widehat{\psi}^{\,\theta}_k-\lambda_k\widehat{\psi}^{\,\theta}_k
=\mathcal{N}(\widehat{\psi}^{\,\theta})_k+r_k^\theta .
\end{align}
Subtracting,
\begin{align}\label{eq:error-ode}
    \partial_t\widehat{e}_k
    =-i\lambda_k\widehat{e}_k
    -i\bigl(\mathcal{N}(\widehat{\psi})_k-\mathcal{N}(\widehat{\psi}^{\,\theta})_k\bigr)
    -i\,a_k+i\,r_k^\theta,
    \qquad k\in\Lambda_K.
\end{align}

\medskip
\noindent\textbf{Energy estimate.}
Differentiating the squared norm and substituting \eqref{eq:error-ode},
\begin{align}
    \frac{1}{2}\frac{d}{dt}\|\widehat{e}(t)\|_{h^s}^2
    &=
    \re\sum_{k\in\Lambda_K}\langle k\rangle^{2s}\,\partial_t\widehat{e}_k\,\overline{\widehat{e}_k}
    \notag\\
    &=
    \underbrace{\re\sum_{k\in\Lambda_K}-i\lambda_k\langle k\rangle^{2s}|\widehat{e}_k|^2}_{=\,0}
    -\re\sum_{k\in\Lambda_K}i\langle k\rangle^{2s}
        \bigl(\mathcal{N}(\widehat{\psi})_k-\mathcal{N}(\widehat{\psi}^{\,\theta})_k\bigr)\overline{\widehat{e}_k}
    \notag\\
    &\quad
    -\re\sum_{k\in\Lambda_K}i\langle k\rangle^{2s}\,a_k\,\overline{\widehat{e}_k}
    +\re\sum_{k\in\Lambda_K}i\langle k\rangle^{2s}\,r_k^\theta\,\overline{\widehat{e}_k}.
\end{align}
The first sum vanishes since $\lambda_k\in\mathbb{R}$ makes each summand purely 
imaginary; this is conservation of the $h^s$ norm under the linear flow. Splitting 
the weight as $\langle k\rangle^{2s}=\langle k\rangle^{s}\cdot\langle k\rangle^{s}$ 
and applying the Cauchy--Schwarz inequality to each remaining sum,
\begin{align}\label{eq:ebound}
    \frac{1}{2}\frac{d}{dt}\|\widehat{e}(t)\|_{h^s}^2
    \le
    \Bigl(
        \bigl\|\mathcal{N}(\widehat{\psi})-\mathcal{N}(\widehat{\psi}^{\,\theta})\bigr\|_{h^s}
        +\|a(t)\|_{h^s}
        +\|r^\theta(t)\|_{h^s}
    \Bigr)\,\|\widehat{e}(t)\|_{h^s}.
\end{align}

\medskip
\noindent\textbf{Nonlinear difference.}
Using \eqref{eq:N-identity} for both arguments and 
$\mathcal{N}(\widehat{\psi}^{\,\theta})=P_{\le K}(|\psi^\theta|^2\psi^\theta)$,
\begin{align}\label{eq:transfer}
    \bigl\|\mathcal{N}(\widehat{\psi})-\mathcal{N}(\widehat{\psi}^{\,\theta})\bigr\|_{h^s}
    &=
    \bigl\|P_{\le K}\bigl(|P_{\le K}\psi|^2 P_{\le K}\psi-|\psi^\theta|^2\psi^\theta\bigr)\bigr\|_{H^s}
    \notag\\
    &\le
    \bigl\||P_{\le K}\psi|^2 P_{\le K}\psi-|\psi^\theta|^2\psi^\theta\bigr\|_{H^s},
\end{align}
since $P_{\le K}$ does not increase the $H^s$ norm. Writing 
$\tilde{e}:=P_{\le K}\psi-\psi^\theta$ and telescoping the cubic difference into 
three terms, each linear in $\tilde{e}$,
\begin{align}\label{eq:telescope}
    |P_{\le K}\psi|^2 P_{\le K}\psi-|\psi^\theta|^2\psi^\theta
    =
    \tilde{e}\,\overline{P_{\le K}\psi}\,P_{\le K}\psi
    +\psi^\theta\,\overline{\tilde{e}}\,P_{\le K}\psi
    +\psi^\theta\,\overline{\psi^\theta}\,\tilde{e}.
\end{align}
Applying the algebra property (Lemma~\ref{lem:algebra}) to each term and combining 
with \eqref{eq:transfer},
\begin{align}
    \bigl\|\mathcal{N}(\widehat{\psi})-\mathcal{N}(\widehat{\psi}^{\,\theta})\bigr\|_{h^s}
    &\le
    \|\tilde{e}\,\overline{P_{\le K}\psi}\,P_{\le K}\psi\|_{H^s}
    +\|\psi^\theta\,\overline{\tilde{e}}\,P_{\le K}\psi\|_{H^s}
    +\|\psi^\theta\,\overline{\psi^\theta}\,\tilde{e}\|_{H^s}
    \notag\\
    &\le
    C_s\Bigl(\|P_{\le K}\psi\|_{H^s}^2
    +\|P_{\le K}\psi\|_{H^s}\|\psi^\theta\|_{H^s}
    +\|\psi^\theta\|_{H^s}^2\Bigr)\|\tilde{e}\|_{H^s}.
\end{align}
Since $P_{\le K}$ is a contraction on $H^s$, $\|P_{\le K}\psi\|_{H^s}\le\|\psi\|_{H^s}$, 
and moreover 
\begin{align}
    \tilde{e}=P_{\le K}\psi-\psi^\theta=P_{\le K}(\psi-\psi^\theta)
\end{align}
has Fourier support in $\Lambda_K$ with $\|\tilde{e}\|_{H^s}=\|\widehat{e}\|_{h^s}$. Using 
$\|P_{\le K}\psi\|_{H^s}\|\psi^\theta\|_{H^s}\le\tfrac12(\|\psi\|_{H^s}^2+\|\psi^\theta\|_{H^s}^2)$, 
absorbing constants into a new $C_s>0$, and invoking \eqref{mequ},
\begin{align}\label{eq:nonlin-bound}
    \bigl\|\mathcal{N}(\widehat{\psi})-\mathcal{N}(\widehat{\psi}^{\,\theta})\bigr\|_{h^s}
    \le
    C_s\bigl(\|\psi(t)\|_{H^s}+\|\psi^\theta(t)\|_{H^s}\bigr)^2\|\widehat{e}(t)\|_{h^s}
    \le
    C_sM^2\|\widehat{e}(t)\|_{h^s}.
\end{align}

\medskip
\noindent\textbf{Truncation term.}
Each triple defining $a_k$ in \eqref{eq:psi-split} has at least one factor indexed 
outside $\Lambda_K$, i.e.\ a factor of $P_{>K}\psi$. By the algebra property 
(Lemma~\ref{lem:algebra}) and then the Bernstein inequality 
(Lemma~\ref{lem:bernstein}),
\begin{align}\label{eq:alias-bound}
    \|a(t)\|_{h^s}
    \le
    C_s\,\|P_{>K}\psi(t)\|_{H^s}\,\|\psi(t)\|_{H^s}^2
    \le
    \frac{C_s}{K}\,\|\psi(t)\|_{H^{s+1}}\,\|\psi(t)\|_{H^s}^2
    \le
    \frac{C_sM^3}{K}.
\end{align}

\medskip
\noindent\textbf{Differential inequality and Grönwall.}
Inserting \eqref{eq:nonlin-bound} and \eqref{eq:alias-bound} into \eqref{eq:ebound} 
and dividing by $\|\widehat{e}(t)\|_{h^s}$ where positive,
\begin{align}
    \frac{d}{dt}\|\widehat{e}(t)\|_{h^s}
    \le
    C_sM^2\|\widehat{e}(t)\|_{h^s}
    +\frac{C_sM^3}{K}
    +\|r^\theta(t)\|_{h^s}.
\end{align}
Grönwall's inequality gives, for every $t\in[0,T]$,
\begin{align}
    \|\widehat{e}(t)\|_{h^s}
    \le
    \exp(C_sM^2 t)
    \left[
        \|\widehat{e}(0)\|_{h^s}
        +\int_0^t\Bigl(\|r^\theta(\tau)\|_{h^s}+\frac{C_sM^3}{K}\Bigr)d\tau
    \right].
\end{align}
Since $e=\psi-\psi^\theta$ and $\|\cdot\|_{h^s}=\|\cdot\|_{H^s}$ under the Fourier 
identification, taking the supremum over $0\le t\le T$ yields
\begin{align}
    \sup_{0\le t\le T}\|\psi(t)-\psi^\theta(t)\|_{H^s}
    \le
    \exp(C_sM^2T)
    \left[
        \|\psi(0)-\psi^\theta(0)\|_{H^s}
        +\int_0^T\|r^\theta(\tau)\|_{h^s}\,d\tau
        +\frac{C_sM^3T}{K}
    \right],
\end{align}
which proves the result.
\end{proof}

\begin{remark}
The assumption 
\begin{align}
    \sup_{t\in[0,T]}\|\psi(t)\|_{H^{s+1}}\le M
\end{align}
is not vacuous. On a generic irrational torus, Theorem~\ref{thm1.1} and 
Corollary~\ref{cor1.2} give polynomial-in-time control of high Sobolev norms for data 
of bounded frequency support, so $M$ may be taken finite (at most linearly growing in 
$T$ for $s\ge3$). In the rational case, \cite[Theorem~1.5]{HK26} furnishes a 
logarithmically-corrected linear approximation on the timescale $[0,T/\log N)$, which 
likewise yields the required norm control there.
\end{remark}

\begin{remark}[Decomposition of the total error]
The above estimate controls the error between the learned solution
\(\psi_\theta\) and the reference solution \(\psi\). If \(\psi^{\mathrm{ex}}\) denotes the exact continuum solution
of \eqref{eq:nls-ivp}, then, for each \(0\le t\le T\),
\[
    \|\psi^{\mathrm{ex}}(t)-\psi^\theta(t)\|_{H^s}
    \le
    \|\psi^{\mathrm{ex}}(t)-\psi(t)\|_{H^s}
    +
    \|\psi(t)-\psi^\theta(t)\|_{H^s} .
\]
Consequently,
\[
    \sup_{0\le t\le T}\|\psi^{\mathrm{ex}}(t)-\psi^\theta(t)\|_{H^s}
    \le
    \sup_{0\le t\le T}\|\psi^{\mathrm{ex}}(t)-\psi(t)\|_{H^s}
    +
   \sup_{0\le t\le T} \|\psi(t)-\psi^\theta(t)\|_{H^s} .
\]
Thus, the total error relative to the continuum solution is decomposed into the reference-solver error and the learned-solution error controlled by the
preceding theorem. In our computations, \(\psi\) is generated by a Fourier pseudo-spectral reference solver. Convergence results for conservative
Fourier pseudo-spectral methods for two-dimensional nonlinear Schrödinger equations are available in \citep{GongWangWangCai2017ConservativeFPSNLS}. The precise rate depends on the spatial truncation, the time integrator, the time step, and the regularity of the exact solution.
\end{remark}

\section{Numerical Experiments}\label{sec: experiment}
In this section, we present numerical experiments for the two-dimensional cubic defocusing NLS on rational and irrational tori. We first describe the data generation procedure, computational setup, model training, and evaluation metrics. We then compare the learned and ground truth solutions, with particular attention to the growth of Sobolev norms, which provides the main numerical evidence supporting the theoretical discussion in Section~\ref{sec: theory}. Finally, we include ablation studies to examine how architectural choices affect the accuracy of the learned dynamics.
\begin{figure}[H]
    \centering
     \includegraphics[width=.5\textwidth]{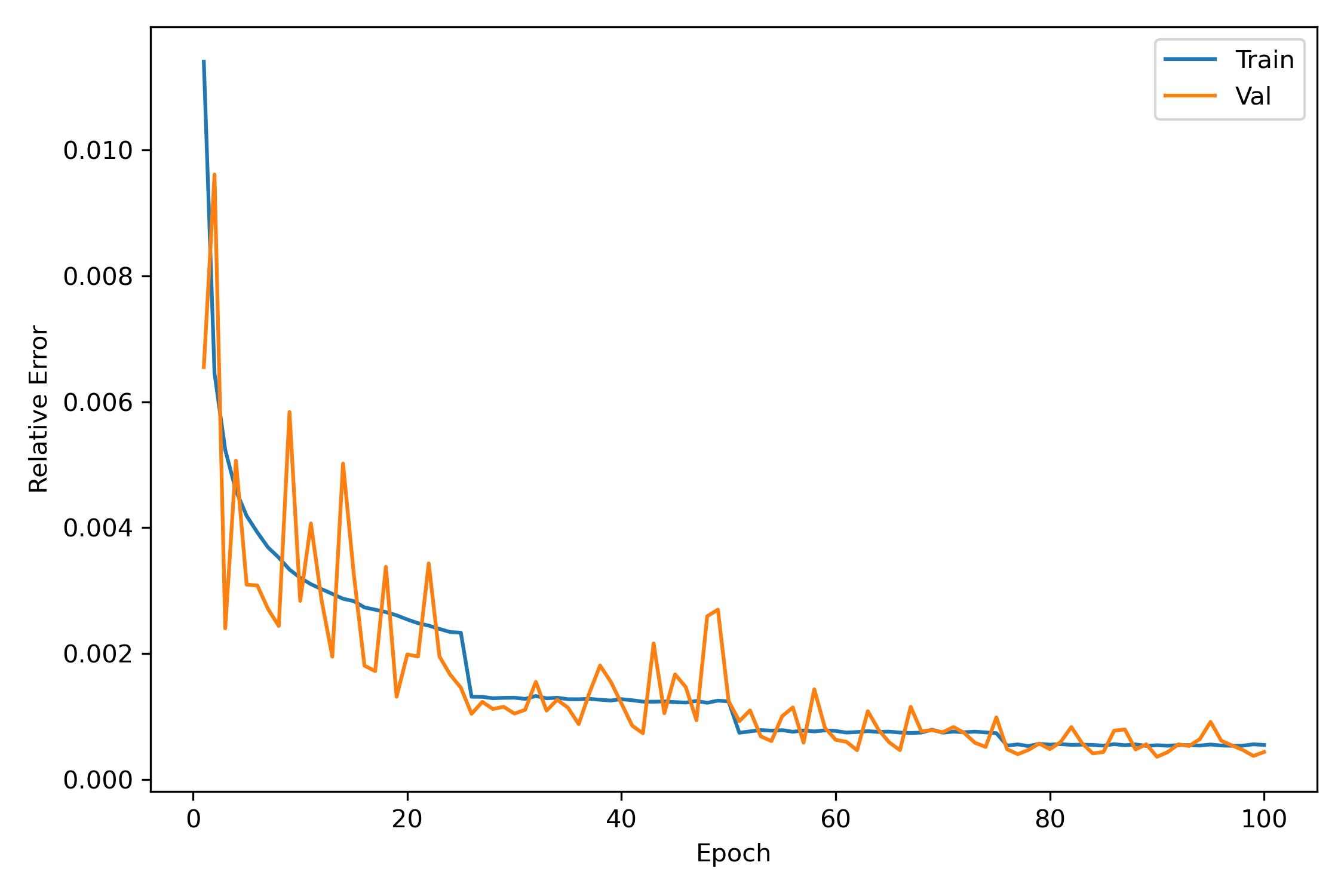}
    \caption{Training and validation relative $L_2$ error as functions of epoch}
    \label{epochlossrnn}
\end{figure}

\subsection{Data Generation and Computational Setup}

The reference data are generated on the computational domain $[0,2\pi]^2$, with the torus geometry entering through the anisotropic Fourier multiplier $\lambda_{m,\ell}=m^2+\omega^2\ell^2.$ We consider the rational case $\omega^2=1$ and the irrational case $\omega^2=\sqrt{2}$. For each realization, the initial condition is constructed as a random-phase Fourier series supported on the low-frequency set
\[
    K_0=\{(m,\ell)\in\mathbb{Z}^2:-2\le m,\ell\le 2,\ (m,\ell)\ne (0,0)\}.
\]
Specifically,
\[
    \psi_0(x,y)
    =
    \sum_{(m,\ell)\in K_0}
    Ce^{i\phi_{m,\ell}}e^{i(mx+\ell y)},
\]
where the phases, $\phi_{m,\ell}\sim\mathcal{U}[0,2\pi],$ are sampled independently. The constant $C$ is chosen so that the initial data have the prescribed anisotropic Sobolev size $ \|\psi_0\|_{H^s}=R$. As in \cite{hrabski2021energy}, we take $s=2$ and $R=1.8263$.

The reference solutions are computed using a Fourier pseudo-spectral method in space and an integrating-factor fourth-order Runge--Kutta method in time. The linear part is advanced exactly in Fourier space, whereas the cubic nonlinear term $|\psi|^2\psi$ is evaluated pseudospectrally in physical space. A two-thirds dealiasing rule \cite{Orszag1971Dealiasing} is applied to the nonlinear term. Unless stated otherwise, the reference solutions are generated on a $256\times256$ spatial grid with time step $\Delta t_{\mathrm{ref}}=2\times 10^{-3}$. These solutions are then downsampled to a $64\times64$ learning grid in which all training, validation, testing, and error computations are performed. 

The solution is stored at time intervals $\Delta t_{\mathrm{data}}=T_f/40,\, T_f = 2\pi,$ and the final time is $T_{\max}=10T_f$. Hence, each trajectory contains $400$ one-step intervals. The training data are formed from consecutive pairs of snapshots. For each time level $n$, the input consists of the real part, the imaginary part, and a constant channel containing the value of $\omega^2$,
\[
    a(x,y)=
    \left(
    \re \psi^n(x,y),
    \im \psi^n(x,y),
    \omega^2
    \right),
\]
and the target is
\[
   u(x,y)=
    \left(
    \re \psi^{n+1}(x,y),
    \im \psi^{n+1}(x,y)
    \right).
\]
Therefore, the learned map is a geometry-conditioned one-step solution operator, $ \mathcal J_{\theta}: \mathcal{A}\longmapsto \mathcal{U}$. Long-time predictions are then obtained by autoregressively feeding the predicted state back into the model.

\subsubsection{Dataset split}
The data are split according to the random initial conditions determined by the sampled phases
$\{\phi_{m,\ell}\}$. Thus, all one-step pairs from the trajectory generated by a given set of phases are assigned to only one of the training, validation, or testing sets. In our experiment, we use $1200$ such samples for each geometry ($\omega^2 \in \{1, \sqrt{2}\}$), partitioned as $800$ for training, $200$ for validation, and $200$ for testing. The model is then trained on the resulting one-step pairs from both geometries.

The real and imaginary components are standardized using statistics computed only from the training set. The $\omega^2$ channel is kept unnormalized so that the model receives the physical value of the torus parameter directly.

\subsubsection{Training and Evaluation}\label{traineval}

Our architecture uses 4 Fourier layers, width $64$, and $12$ retained Fourier modes in each spatial direction. The lifting map $P$ sends $a(x)$ to a $64$-channel latent representation. Each Fourier layer consists of a truncated spectral convolution and a pointwise $1\times 1$ convolution in physical space, followed by a GELU activation. The projection map $Q$ sends the final latent representation through $128$ channels to the two-component output $u(x)$.

Training uses Adam with initial learning rate $10^{-3}$, mini-batch size $16$, and mean squared error loss on the normalized real and imaginary components. We train for $100$ epochs and reduce the learning rate by a factor of $0.5$ every $25$ epochs. The checkpoint with the smallest validation loss is used for testing. Test performance is reported from full autoregressive rollouts on unseen realizations using relative $L^2$ error for $\psi$, comparisons of $|\psi|$, solution slices, and the Sobolev norm growth $\|\psi(t)\|_{H^2}$.  We conduct all our experiments on a single NVIDIA A40 GPU with $48$ GB memory, using $8$ CPU cores and $128$ GB RAM.

\subsection{Numerical Results}

In this section, we present the numerical results for the learned one-step solution operator on the rational and irrational tori. We include training and validation diagnostics, relative error comparisons, solution-level visualizations, and Sobolev-norm plots. 

\begin{figure}[H]
    \centering

    \begin{minipage}[c]{0.48\textwidth}
        \centering

        \begin{subfigure}[b]{\textwidth}
            \centering
            \includegraphics[width=\textwidth]{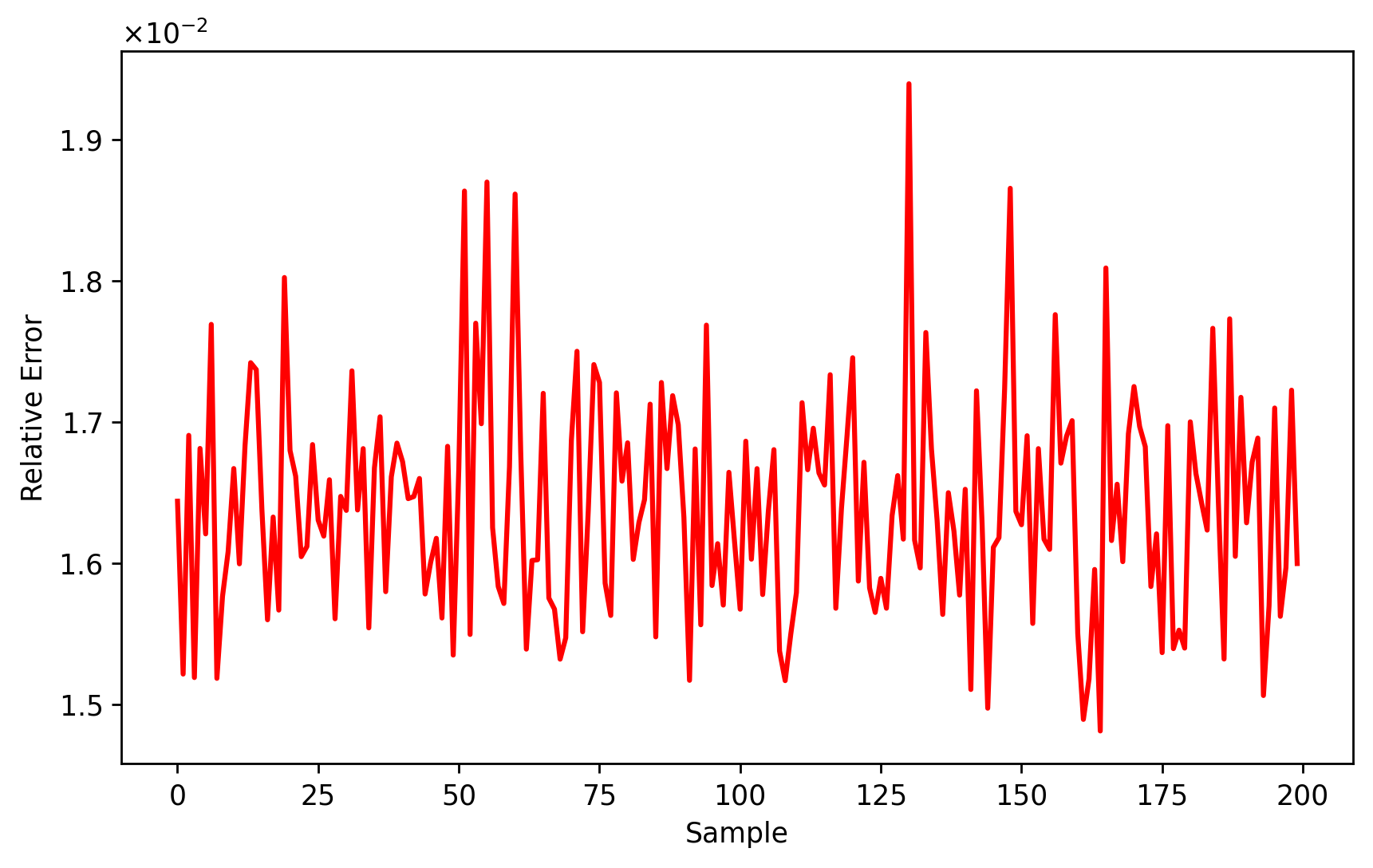}
            \caption{}
            \label{rel2rat}
        \end{subfigure}

        \vspace{0.6em}

        \begin{subfigure}[b]{\textwidth}
            \centering
            \includegraphics[width=\textwidth]{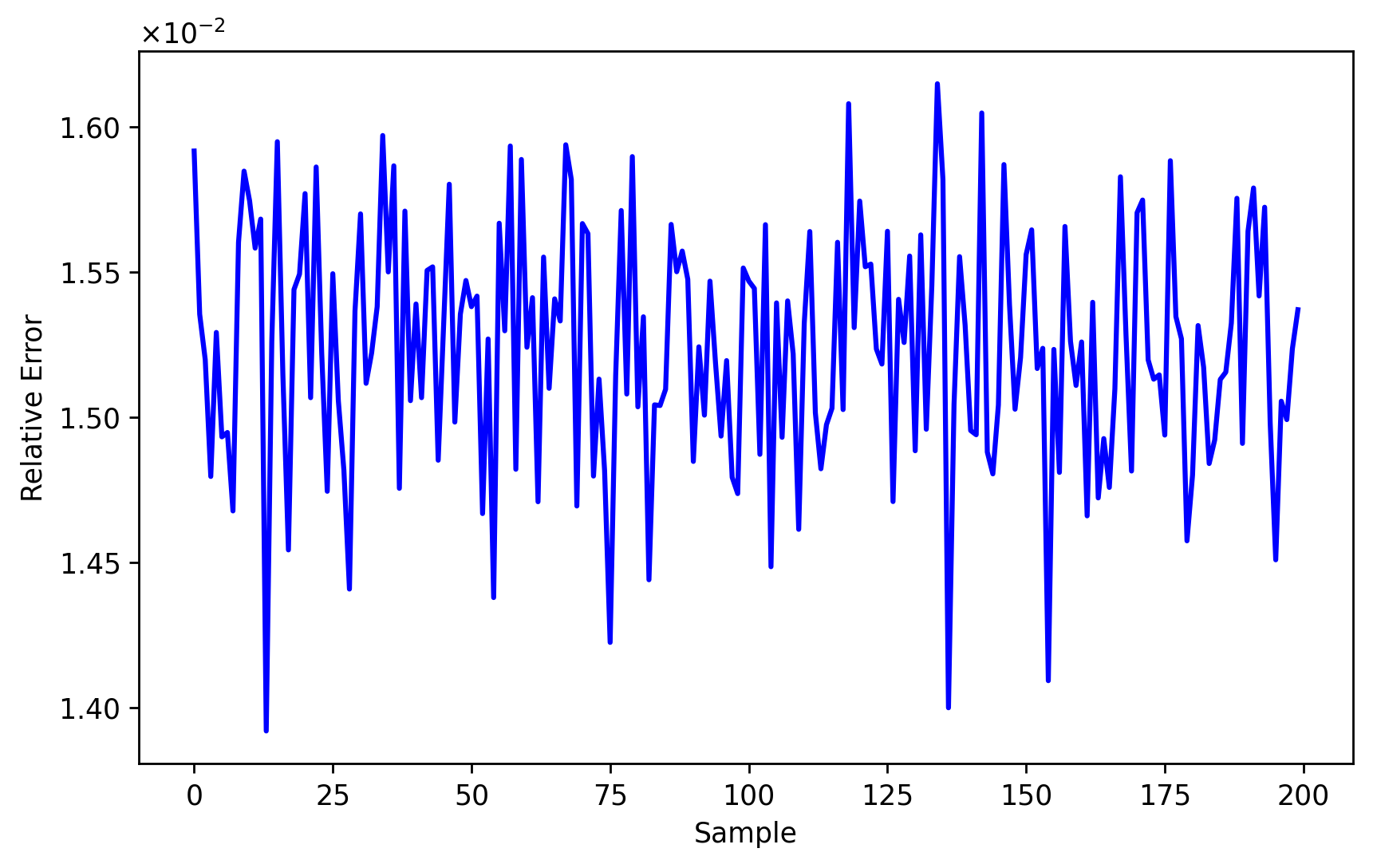}
            \caption{}
            \label{rel2irr}
        \end{subfigure}

    \end{minipage}
    \hfill
    \begin{minipage}[c]{0.50\textwidth}
        \centering

        \begin{subfigure}[b]{\textwidth}
            \centering
            \includegraphics[width=\textwidth]{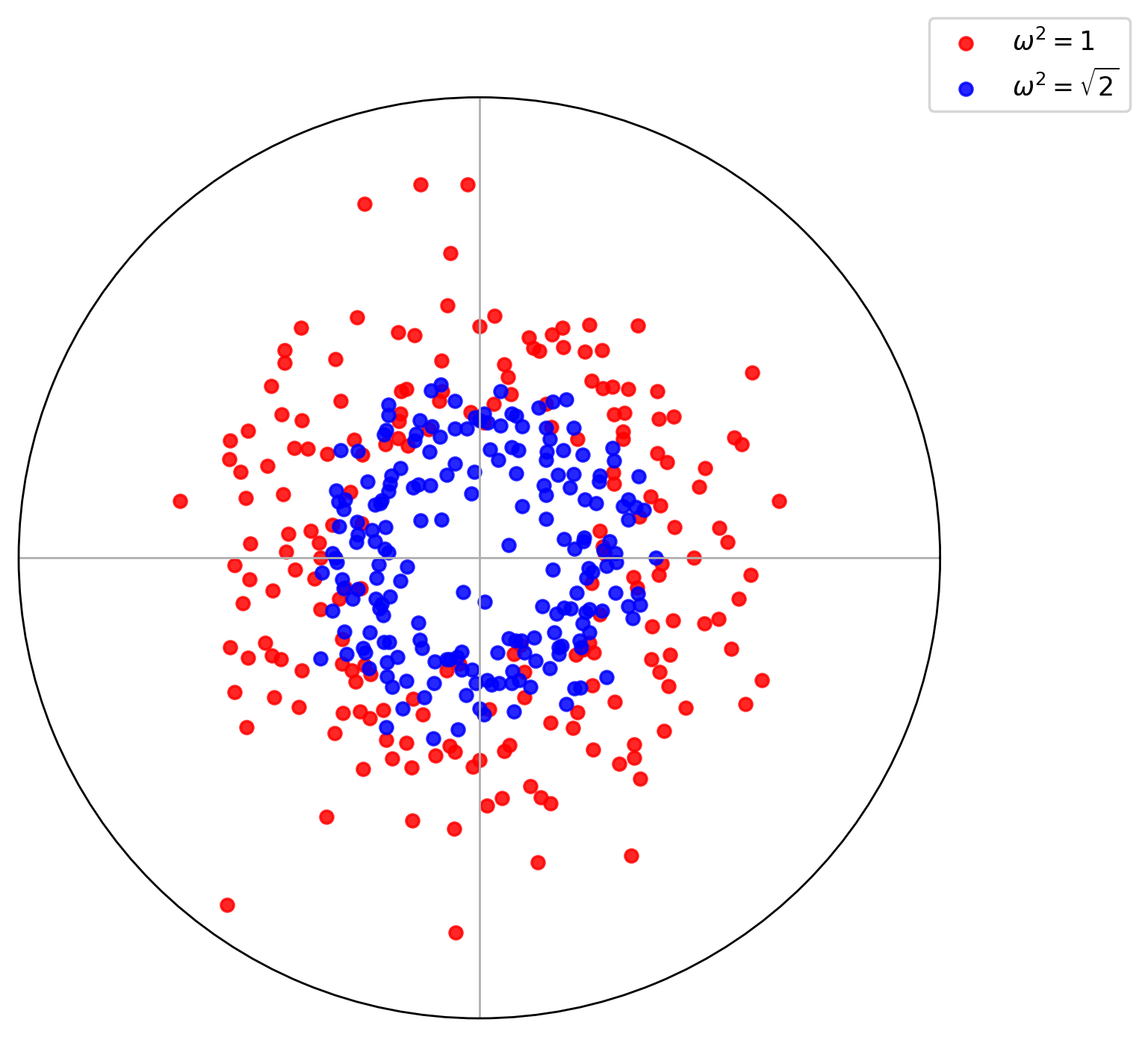}
            \caption{}
            \label{rel2polar}
        \end{subfigure}

    \end{minipage}

    \caption{Network testing. 
    (a) Relative \(L_2\) error of the solution obtained on the rational tori in the test set; 
    (b) Relative \(L_2\) error of the solution obtained on the irrational tori in the test set; 
    (c) Polar representation of the test-set relative \(L_2\) errors.}
    \label{trainrnn}
\end{figure}
\begin{figure}[H]
    \centering

    \begin{subfigure}[b]{0.48\textwidth}
        \centering
        \includegraphics[width=\textwidth]{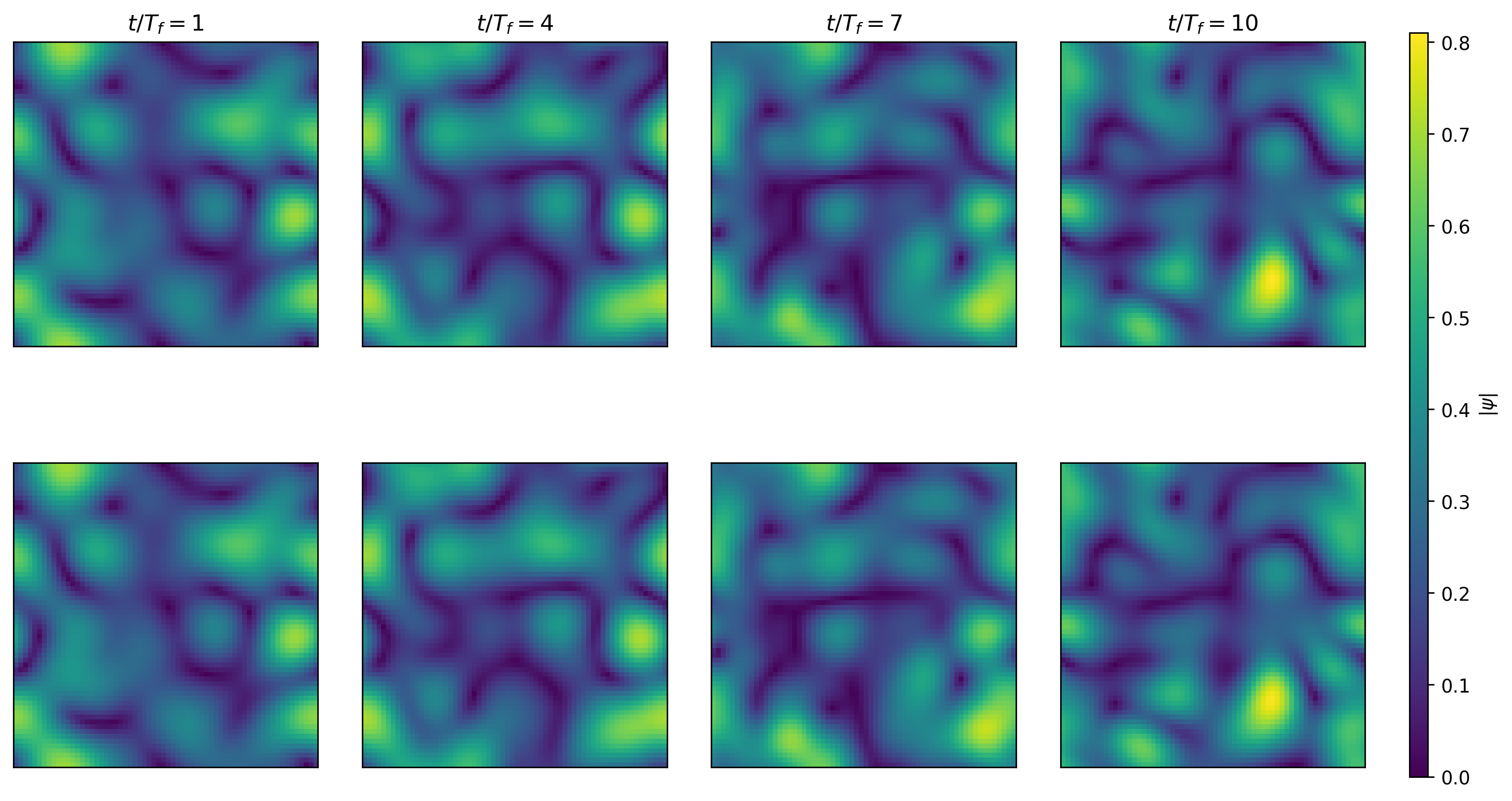}
        \caption{}
        \label{snapshotrat}
    \end{subfigure}
    \hfill
    \begin{subfigure}[b]{0.48\textwidth}
        \centering
        \includegraphics[width=\textwidth]{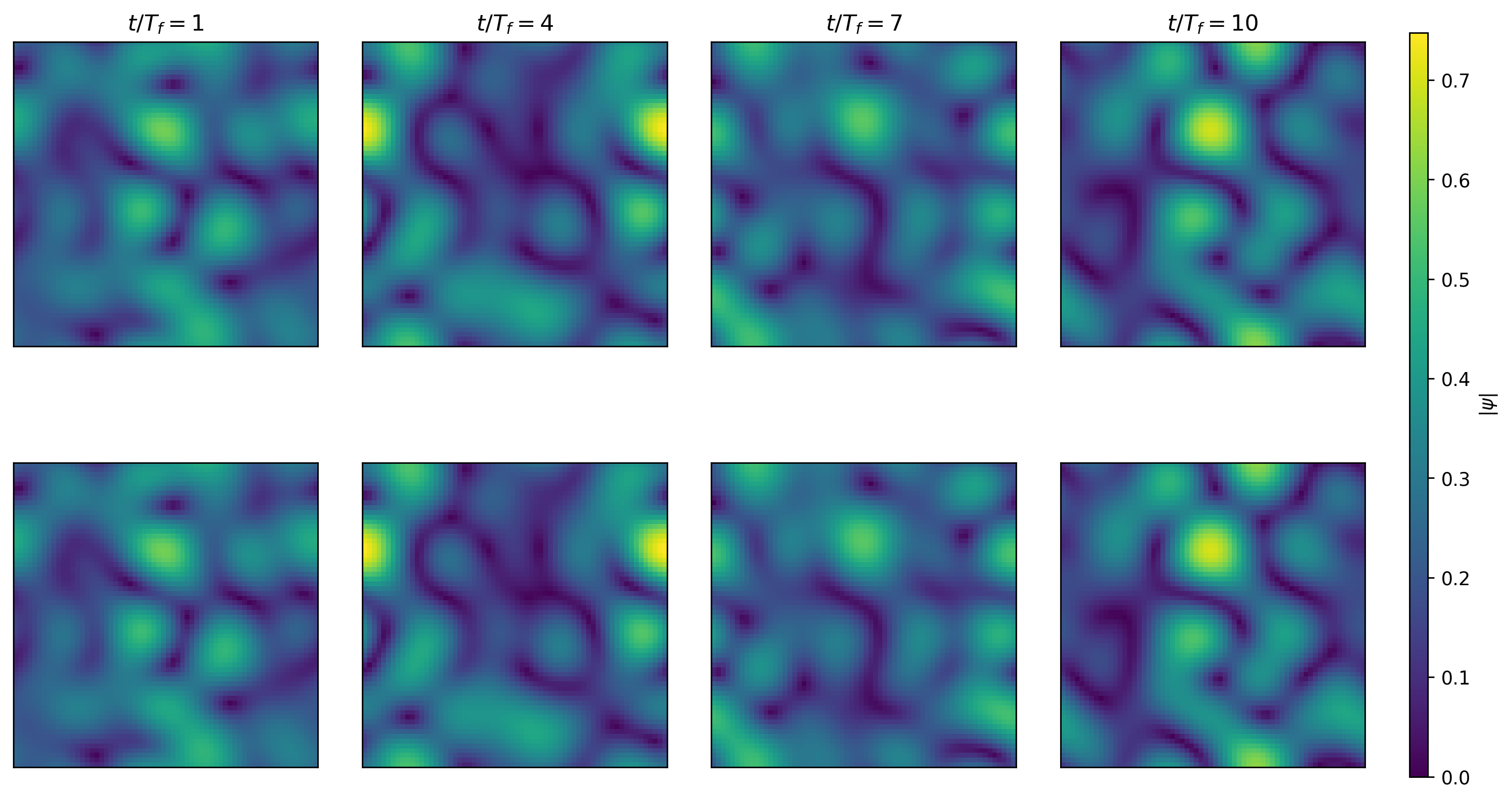}
        \caption{}
        \label{snapshotirr}
    \end{subfigure}   

    \caption{Snapshots of $|\psi|$ on rational and irrational tori. For both cases (a) and (b), the top row shows the ground truth solution and the bottom row shows the prediction at $t/T_f=1,4,7,10$.}
    \label{snapshotsol}
\end{figure}

In Figure~\ref{epochlossrnn}, the training and validation errors decrease steadily over the epochs, with the validation curve closely following the training curve. This suggests that the model learns the one-step solution operator without a clear sign of overfitting. The sample-wise relative \(L^2\) errors in Figures~\ref{trainrnn}(a) and~\ref{trainrnn}(b) remain on the order of \(10^{-2}\) across the test set for both the rational and irrational tori, showing that the learned operator gives consistent accuracy over unseen initial conditions. Figure~\ref{trainrnn}(c) gives a compact polar representation of test-set errors. For each test sample \(j=1,\ldots,200\), we plot $ z_j = r_j e^{i\vartheta_j},$ where the radius $r_j$ is defined as the relative $L_2$ error, and the angle $\vartheta={2\pi(j-1)}/200$ indexes the test samples uniformly around the circle. Thus, points closer to the center correspond to smaller errors, while points farther from the center correspond to larger errors. The rational case shows a wider radial spread, whereas the irrational case is more concentrated toward smaller errors.

\begin{figure}[H]
    \centering
    
    \begin{subfigure}[b]{0.48\textwidth}
        \centering
        \includegraphics[width=\textwidth]{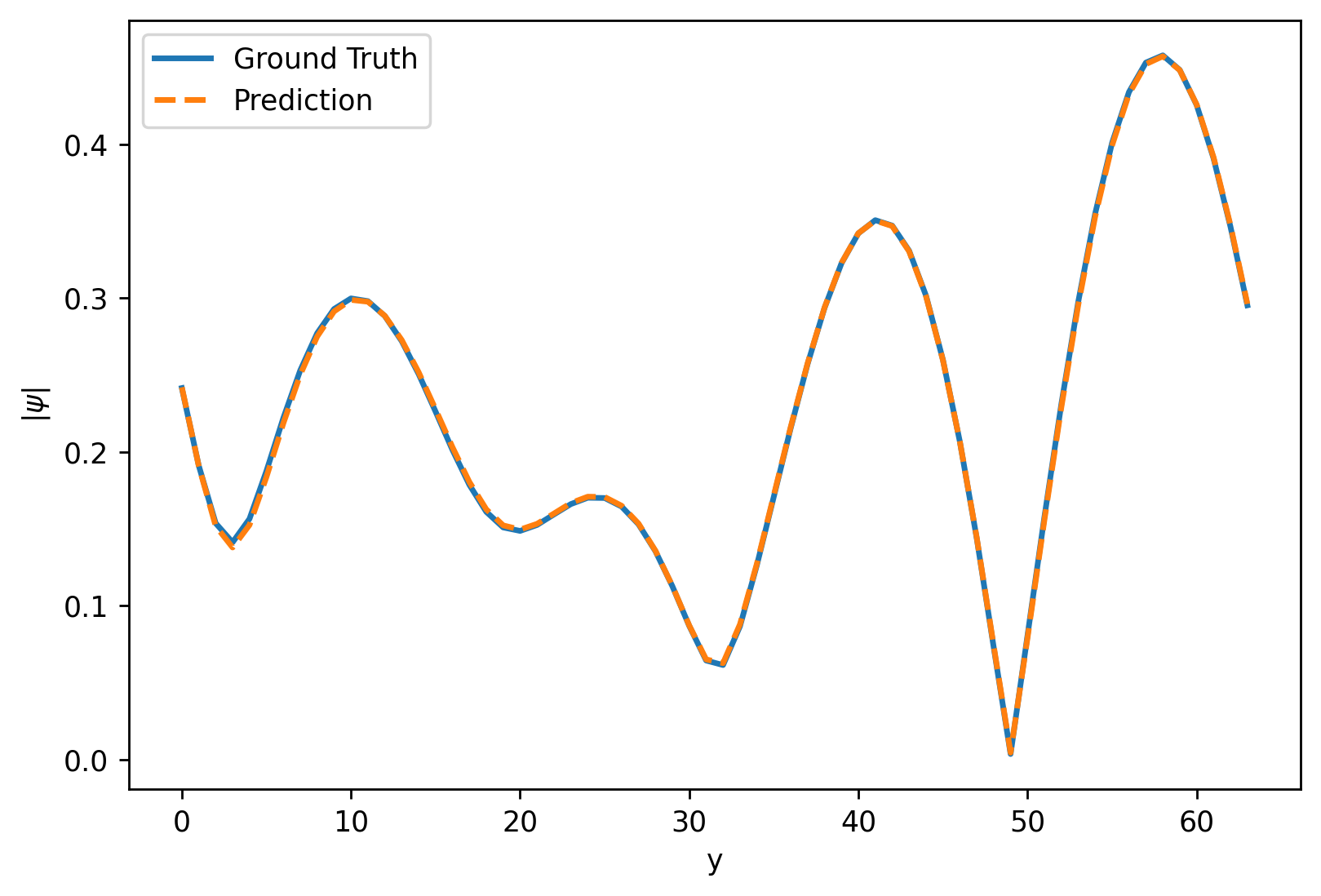}
        \caption{}
        \label{slicerat}
    \end{subfigure}
    \hfill 
    \begin{subfigure}[b]{0.48\textwidth}
        \centering
        \includegraphics[width=\textwidth]{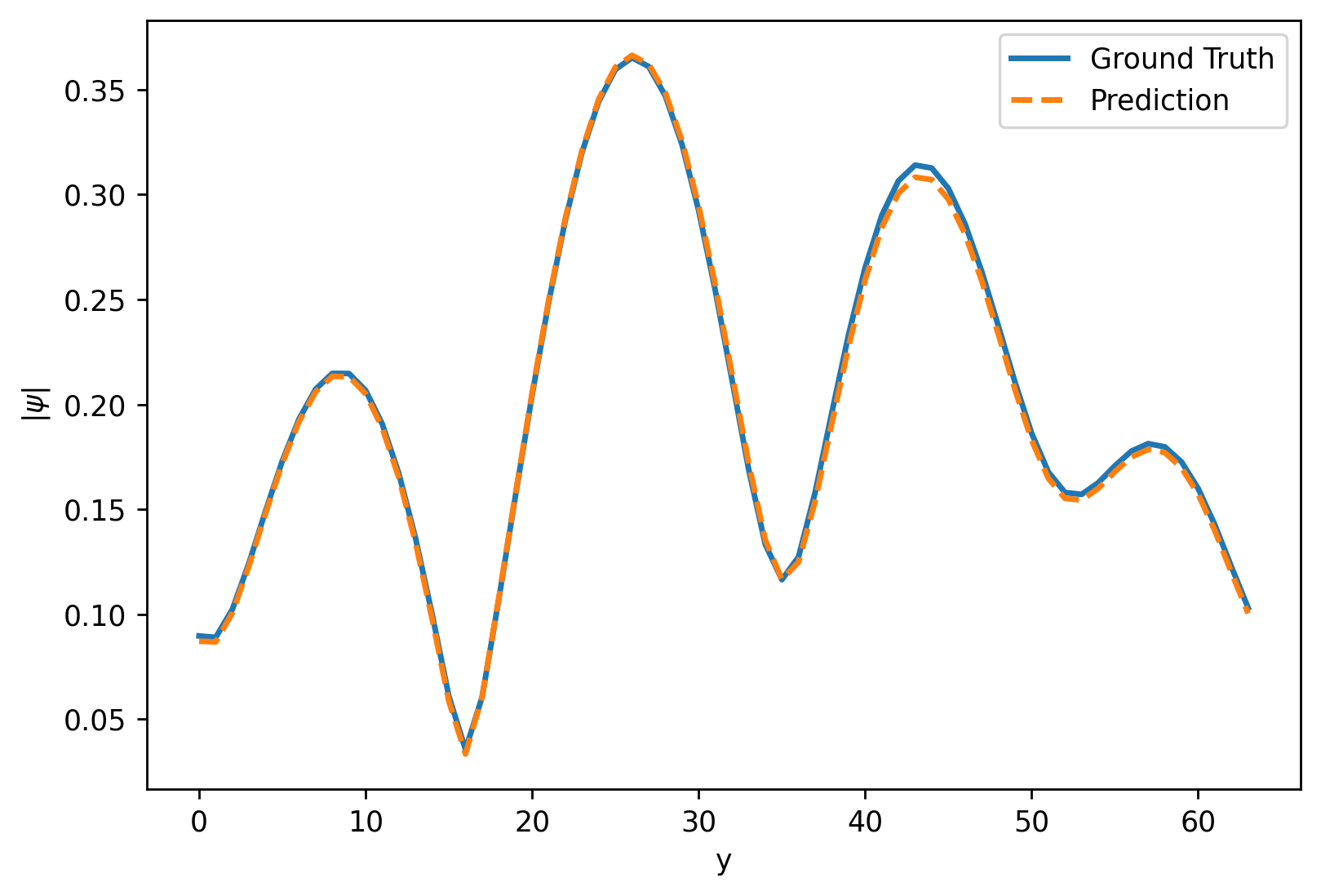}
        \caption{}
        \label{sliceirr}
    \end{subfigure}

   \caption{Comparison chart of the ground truth and predicted solutions for $|\psi|$ at $t/T_f=2$ on rational in (a) and irrational tori in (b).}
    \label{slicesol}
\end{figure}

Figure~\ref{snapshotsol} shows that the predicted solution agrees closely with the ground truth on both the rational and irrational tori. The learned solution operator captures the main spatial patterns and their evolution without noticeable distortion in the plotted time window. In Figure~\ref{slicesol}, we plot one-dimensional slice profiles of $|\psi|$ at $t/T_f=2$ against the grid index on the $64\times64$ data grid, which further confirms the agreement, since the predicted and ground truth curves are nearly indistinguishable for both geometries.

\begin{figure}[H]
    \centering
    \includegraphics[width=0.5\linewidth]{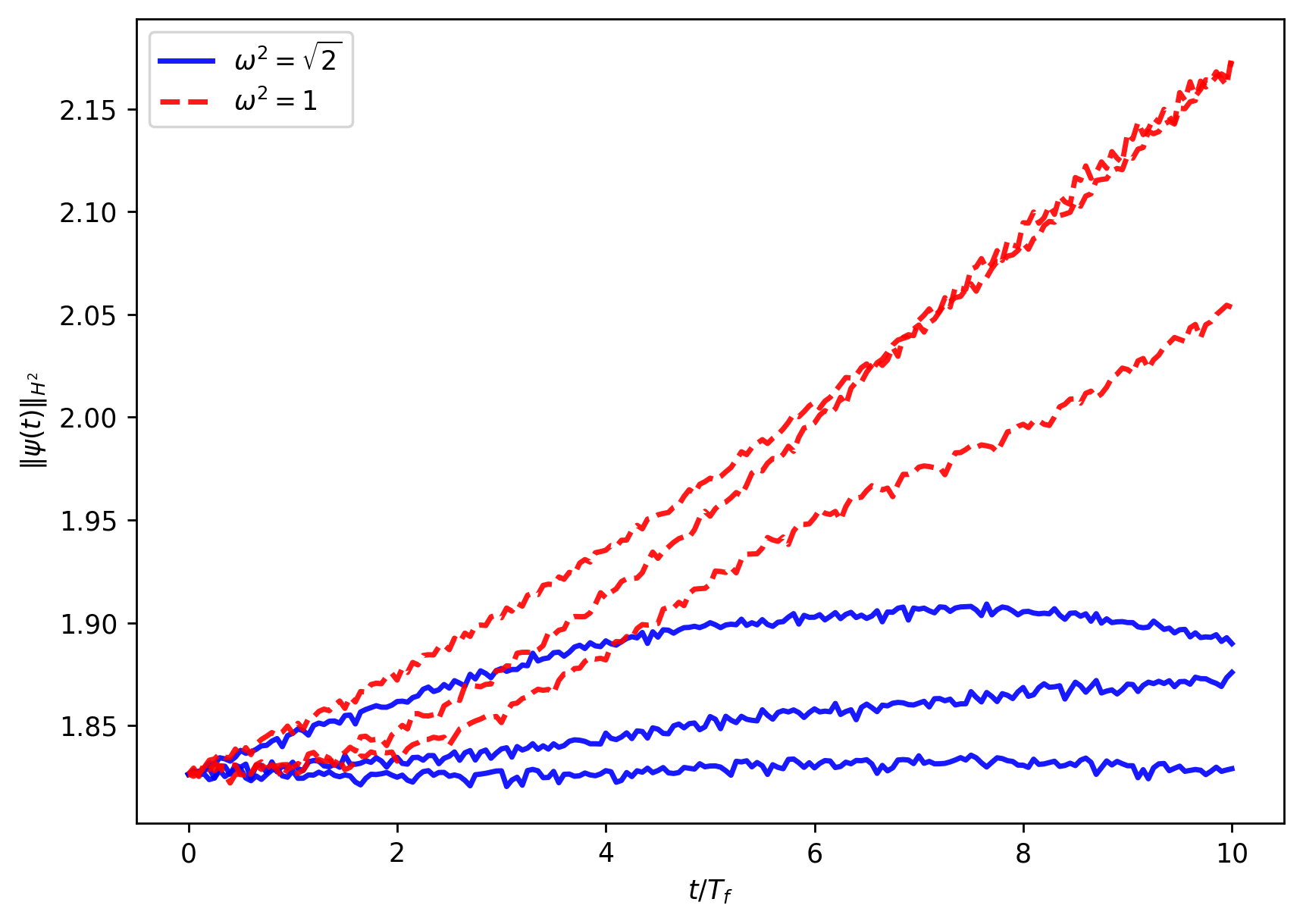}
    \caption{The growth of $\|\psi\|_{H^2}$ in the rational and irrational tori.}
    \label{Hscomparison}
\end{figure}
In Figure~\ref{Hscomparison}, we plot the Sobolev norm growth computed from the learned operator for three test realizations from each torus geometry. The rational case \((\omega^2=1)\), shown with dashed curves, exhibits stronger growth in \(\|\psi(t)\|_{H^2}\) over the time interval, while the irrational case \((\omega^2=\sqrt{2})\), shown with solid curves, shows more constrained growth and remains lower overall. This behavior is consistent with the theoretical discussion in Section~\ref{sec: theory}, where the irrational geometry restricts the resonance structure and leads to weaker transfer of energy to higher Fourier modes. In particular, the numerical result is consistent with Theorem~\ref{thm1.1} and Corollary~\ref{cor1.2}, which indicate that the growth of the Sobolev norm is more constrained on sufficiently irrational tori. Thus, the model captures the geometry-dependent spectral behavior predicted by the analysis.

\subsection{Ablation Study of Critical Factors in the FNO Scheme}
To assess the sensitivity of our operator to key architectural choices, we perform ablation studies on the number of retained Fourier modes, the nonlinear activation function, the Fourier-layer depth, and the use of explicit geometry conditioning. 

For a test sample indexed by $i$, let $\psi_{i}^n$ denote the ground truth solution at the time level $n$, and let $\psi_{i,\theta}^n$ denote the corresponding prediction. The relative $L^2$ error is defined by
\[
\mathcal{E}_i
=
\left(
\frac{
\displaystyle\sum_{n=1}^{M}\sum_{j,k}
\left|
\psi_{i,\theta}^n(x_j,y_k)
-
\psi_{i}^n(x_j,y_k)
\right|^2
}{
\displaystyle\sum_{n=1}^{M}\sum_{j,k}
\left|
\psi_{i}^n(x_j,y_k)
\right|^2
}
\right)^{1/2}.
\]
For each ablation setting, we compute $\mathcal E_i$ over all test realizations and report
\[
\mathrm{Worst}=\max_{1\le i\le N} \mathcal E_i,
\quad
\mathrm{Mean}=\frac{1}{N}\sum_{i=1}^{N} \mathcal E_i,
\quad
\mathrm{Best}=\min_{1\le i\le N} \mathcal E_i,
\]
where $N$ is the total number of test samples. These values summarize the least accurate, average, and most accurate test predictions over the full prediction interval for each model configuration.

\subsubsection{The influence of the number of Fourier modes}
In this section, we investigate on how the number of retained Fourier modes affects the accuracy of the learned model. In the FNO architecture, the spectral convolution is computed only on a truncated set of low Fourier modes, while the remaining high-frequency modes are not directly parameterized in Fourier space. Thus, the value of $K$ controls how much spectral information is retained in each Fourier layer. In this experiment, we keep the network width, number of Fourier layers, activation function, training data, and optimizer fixed, and vary only the number of retained Fourier modes.

\begin{table}[H]
\centering
\caption{The relative $L_2$ error corresponding to different Fourier modes, $K$.}
\label{tab:modes_geometry_error}
\begin{tabular*}{0.85\textwidth}{@{\extracolsep{\fill}} c c c c c}
\toprule
\multirow{2}{*}{Fourier Modes $(K)$} 
& \multirow{2}{*}{Geometry $(\omega^2)$} 
& \multicolumn{3}{c}{Error} \\
\cmidrule(lr){3-5}
& & Worst & Mean & Best \\
\midrule
\multirow{2}{*}{8}
&  $\sqrt{2}$ & 0.0551 & 0.0392 & 0.0254 \\
&  $1$        & 0.1091 & 0.0848 & 0.0610 \\
\midrule
\multirow{2}{*}{12}
&  $\sqrt{2}$ & 0.0558 & 0.0366 & 0.0251 \\
&  $1$        & 0.0954 & 0.0625 & 0.0267 \\
\midrule
\multirow{2}{*}{16}
&  $\sqrt{2}$ & 0.1074 & 0.0870 & 0.0748 \\
&  $1$        & 0.1075 & 0.0883 & 0.0689 \\
\midrule
\multirow{2}{*}{20}
&  $\sqrt{2}$ & 0.0897 & 0.0676 & 0.0473 \\
&  $1$        & 0.1239 & 0.0993 & 0.0647 \\
\bottomrule
\end{tabular*}
\end{table}
Table~\ref{tab:modes_geometry_error} shows that increasing the number of retained Fourier modes does not lead to a monotone improvement in accuracy. The best mean error is obtained with $K=12$ for both geometries, with mean errors $0.0366$ for the irrational torus and $0.0625$ for the rational torus. The case $K=8$ is competitive for the irrational geometry, but gives a noticeably larger error for the rational case. Increasing to $K=16$ worsens the performance for both geometries, while $K=20$ improves over $K=16$ but still does not outperform $K=12$. This suggests that retaining too few modes may limit the model's ability to represent the solution dynamics, while retaining too many modes can increase the number of trainable spectral parameters and make optimization less effective for the fixed training setup. Altogether, $K=12$ gives the best balance between spectral resolution and stable learning in this experiment.

\subsubsection{The influence of nonlinear activation functions}

We examine the effect of the nonlinear activation function, with the retained Fourier modes fixed at $K=12$. The nonlinear activation function affects how the network combines local and spectral features across the network layers. It allows the model to represent nonlinear solution behavior beyond a purely linear transformation of the input modes. Its derivative also influences gradient propagation during training, which can affect both convergence and long-time stability. For this reason, the choice of activation function can have a noticeable effect on the accuracy of the learned solution operator. Here, we compare Sigmoid, GELU, Tanh, Swish, and ReLU activations using the same training and testing setup.

\begin{table}[H]
\centering
\caption{The relative $L_2$ error corresponding to different nonlinear activation functions.}
\label{tab:activation_geometry_error}
\begin{tabular*}{0.85\textwidth}{@{\extracolsep{\fill}} l c c c c}
\toprule
\multirow{2}{*}{Activation Function} 
& \multirow{2}{*}{Geometry $(\omega^2)$} 
& \multicolumn{3}{c}{Error} \\
\cmidrule(lr){3-5}
& & Worst & Mean & Best \\
\midrule
\multirow{2}{*}{Sigmoid$(x)$}
&  $\sqrt{2}$ & 0.0302 & 0.0225 & 0.0171 \\
&  $1$        & 0.0446 & 0.0294 & 0.0213 \\
\midrule
\multirow{2}{*}{GELU$(x)$}
&  $\sqrt{2}$ & 0.0558 & 0.0366 & 0.0251 \\
&  $1$        & 0.0954 & 0.0625 & 0.0267 \\
\midrule
\multirow{2}{*}{Tanh$(x)$}
&  $\sqrt{2}$ & 0.0643 & 0.0486 & 0.0380 \\
&  $1$        & 0.1110 & 0.0665 & 0.0406 \\
\midrule
\multirow{2}{*}{Swish$(x)$}
&  $\sqrt{2}$ & 0.1334 & 0.1102 & 0.0946 \\
&  $1$        & 0.1469 & 0.0931 & 0.0632 \\
\midrule
\multirow{2}{*}{ReLU$(x)$}
&  $\sqrt{2}$ & 0.1342 & 0.0899 & 0.0597 \\
&  $1$        & 0.2712 & 0.2153 & 0.1617 \\
\bottomrule
\end{tabular*}
\end{table}

From Table~\ref{tab:activation_geometry_error}, we observe that Sigmoid gives the smallest mean error for both geometries, with mean errors of $0.0225$ for $\omega^2=\sqrt{2}$ and $0.0294$ for $\omega^2=1$. GELU gives the next best performance, with mean errors of $0.0366$ and $0.0625$ for the irrational and rational geometries, respectively. Tanh performs slightly worse than GELU but remains more accurate than Swish and ReLU. ReLU is especially less accurate on the rational torus, where the mean error increases to $0.2153$. Taken together, Sigmoid gives the most accurate activation choice in this experiment.

\subsubsection{The influence of Fourier-layer depth}

The Fourier layers are the main spectral components of the FNO architecture. They map the input representation into Fourier space, apply trainable weights to a fixed number of retained modes, and return the result to physical space. Their depth therefore determines how many times the network processes and recombines spectral information before producing the next predicted state. To examine how this affects the model performance, we vary the number of Fourier layers while keeping the Fourier modes, activation function, training data, and optimization settings fixed. The depths considered are $2$, $4$, $6$, $8$, and $10$.

\begin{table}[H]
\centering
\caption{The relative $L_2$ error corresponding to different Fourier-layer depths.}
\label{tab:depth_geometry_error}
\begin{tabular*}{0.85\textwidth}{@{\extracolsep{\fill}} c c c c c}
\toprule
\multirow{2}{*}{Fourier Layers} 
& \multirow{2}{*}{Geometry $(\omega^2)$} 
& \multicolumn{3}{c}{Error} \\
\cmidrule(lr){3-5}
& & Worst & Mean & Best \\
\midrule
\multirow{2}{*}{2}
&  $\sqrt{2}$ & 0.0173 & 0.0118 & 0.0082 \\
&  $1$        & 0.0541 & 0.0372 & 0.0261 \\
\midrule
\multirow{2}{*}{4}
&  $\sqrt{2}$ & 0.0558 & 0.0366 & 0.0251 \\
&  $1$        & 0.0954 & 0.0625 & 0.0267 \\
\midrule
\multirow{2}{*}{6}
&  $\sqrt{2}$ & 0.1217 & 0.0820 & 0.0423 \\
&  $1$        & 0.1382 & 0.0865 & 0.0450 \\
\midrule
\multirow{2}{*}{8}
&  $\sqrt{2}$ & 0.0780 & 0.0617 & 0.0481 \\
&  $1$        & 0.1181 & 0.0740 & 0.0548 \\
\midrule
\multirow{2}{*}{10}
&  $\sqrt{2}$ & 0.0900 & 0.0717 & 0.0536 \\
&  $1$        & 0.1226 & 0.0808 & 0.0507 \\
\bottomrule
\end{tabular*}
\end{table}

The prediction accuracy for different Fourier-layer depths is shown in Table~\ref{tab:depth_geometry_error}. The results show that Fourier-layer depth has a clear effect on performance, but the trend is not monotone. The two-layer model gives the smallest mean error for both geometries, with mean errors of $0.0118$ for $\omega^2=\sqrt{2}$ and $0.0372$ for $\omega^2=1$. The four-layer model gives the next best performance, with mean errors of $0.0366$ and $0.0625$ for the irrational and rational geometries, respectively. Increasing the depth further does not improve the accuracy. In particular, the six-layer model gives the largest mean error for both geometries, while the eight- and ten-layer models improve upon the six-layer case but still do not outperform the two- or four-layer architectures. Thus, although additional Fourier layers increase the spectral processing capacity of the network, deeper architectures do not necessarily yield better long-time predictions under the present training setup. These results indicate that Fourier-layer depth must be selected carefully, since excessive depth may increase optimization difficulty and error accumulation.

\subsubsection{The influence of geometry conditioning}

In this section, we compare the geometry-conditioned model
\[
    \mathcal J_{\theta}
    :
    \left(
    \re \psi^n(x,y),
    \im \psi^n(x,y),
    \omega^2
    \right)
    \longmapsto
    \left(
    \re\psi^{n+1}(x,y),
    \im\psi^{n+1}(x,y)
    \right),
\]
with an unconditioned model
\[
    \mathcal J_{\theta}
    :
    \left(
    \re \psi^n(x,y),
    \im\psi^n(x,y)
    \right)
    \longmapsto
    \left(
    \re \psi^{n+1}(x,y),
    \im \psi^{n+1}(x,y)
    \right).
\]
\begin{table}[H]
\centering
\caption{The relative $L_2$ error for geometry-conditioned and unconditioned models.}
\label{tab:geometry_conditioning_error}
\begin{tabular*}{0.85\textwidth}{@{\extracolsep{\fill}} l c c c c}
\toprule
\multirow{2}{*}{Model} 
& \multirow{2}{*}{Geometry $(\omega^2)$} 
& \multicolumn{3}{c}{Error} \\
\cmidrule(lr){3-5}
& & Worst & Mean & Best \\
\midrule
\multirow{2}{*}{GC-FNO}
&  $\sqrt{2}$ & 0.0558 & 0.0366 & 0.0251 \\
&  $1$        & 0.0954 & 0.0625 & 0.0267 \\
\midrule
\multirow{2}{*}{FNO}
&  $\sqrt{2}$ & 0.0618 & 0.0467 & 0.0288 \\
&  $1$        & 0.1358 & 0.1034 & 0.0676 \\
\bottomrule
\end{tabular*}
\end{table}
We investigate whether the learned solution operator can infer the underlying geometry from solution snapshots alone, or whether providing the torus parameter $\omega^2$ improves the accuracy of the model. In the conditioned case, the input includes an additional constant channel containing the value of $\omega^2$, so that a single network can explicitly distinguish between the rational and irrational geometries. In the unconditioned case, this channel is removed, and the network only receives the real and imaginary parts of the solution. 

\begin{figure}
    \centering
    \begin{subfigure}{0.48\textwidth}
        \centering
        \includegraphics[width=\textwidth]{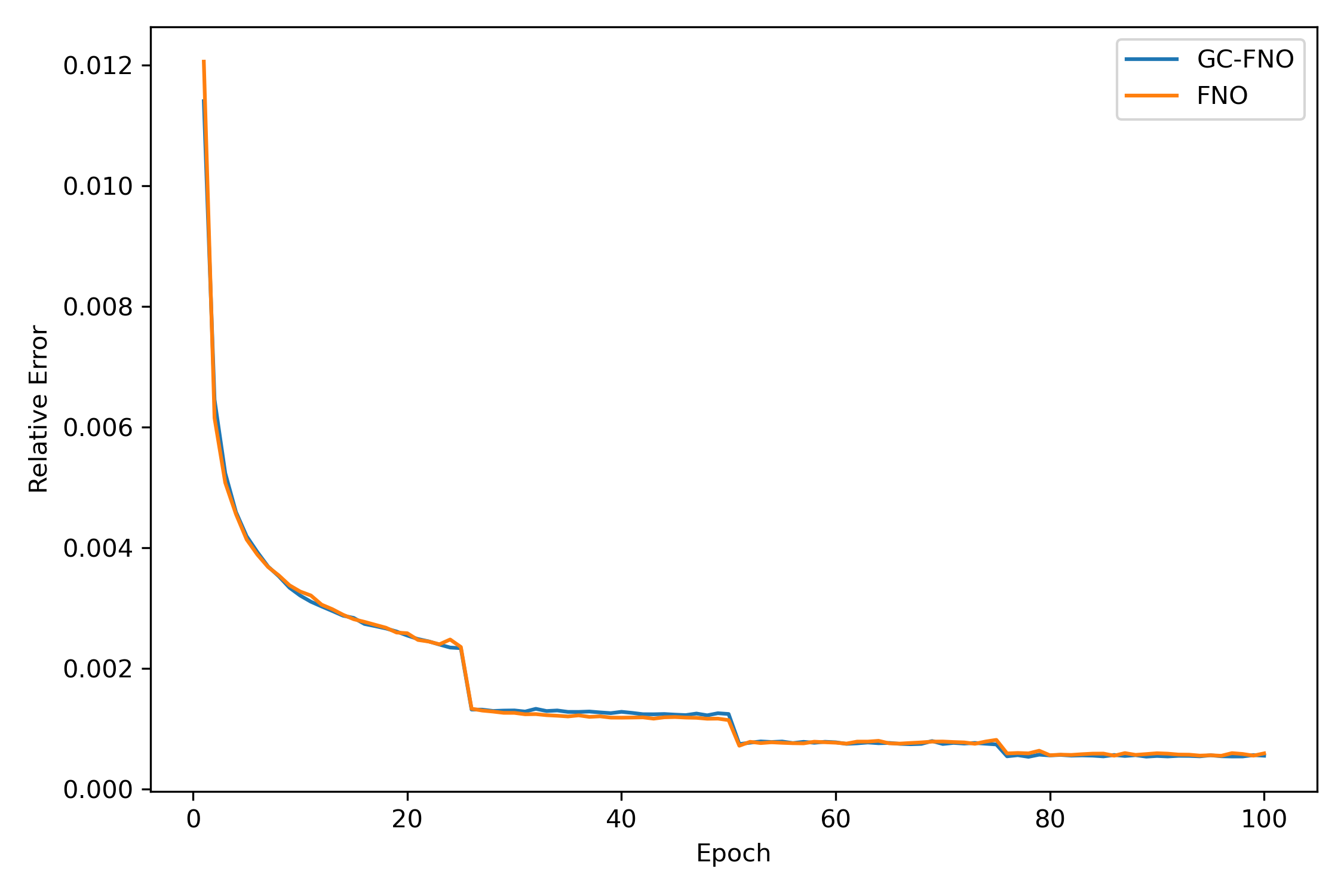}
        \caption{Training relative error.}
        \label{fig:train_conditioned_unconditioned}
    \end{subfigure}
    \hfill
    \begin{subfigure}{0.48\textwidth}
        \centering
        \includegraphics[width=\textwidth]{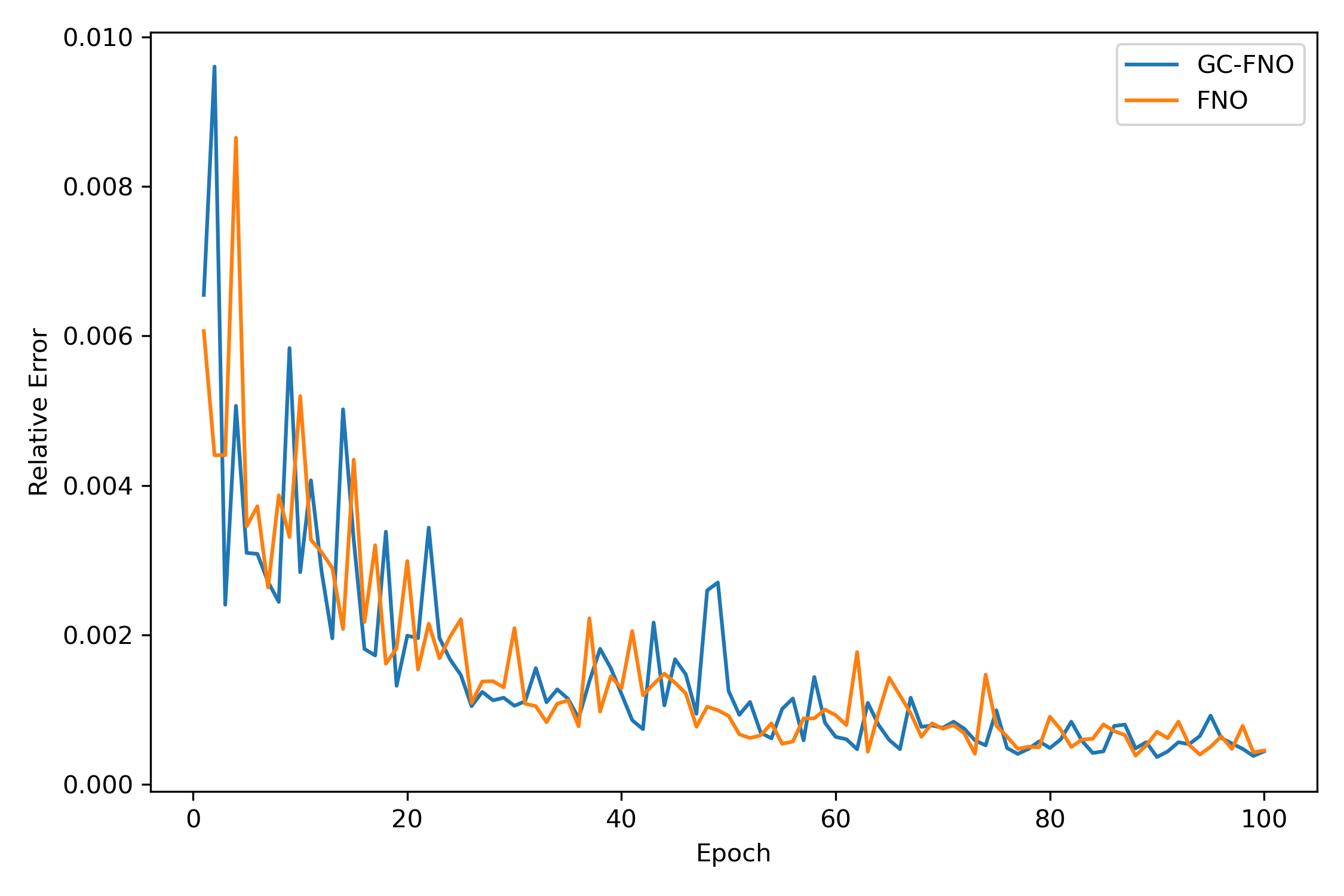}
        \caption{Validation relative error.}
        \label{fig:val_conditioned_unconditioned}
    \end{subfigure}
    \caption{Training and validation relative errors for the geometry-conditioned and unconditioned models.}
    \label{fig:conditioned_unconditioned_training}
\end{figure}

Figures~\ref{fig:train_conditioned_unconditioned}--\ref{fig:val_conditioned_unconditioned} show that both the conditioned and unconditioned models train stably. The training curves are nearly indistinguishable, whereas the validation curves are noisier but decay to similar magnitudes by the end of training. However, the errors in Table~\ref{tab:geometry_conditioning_error} show a clearer benefit from geometry conditioning. For the irrational torus, the conditioned model reduces the mean error from $0.0467$ to $0.0366$. For the rational torus, the improvement is larger, with the mean error decreasing from $0.1034$ to $0.0625$. The worst and best errors also decrease under conditioning for both geometries. Thus, adding the $\omega^2$ channel helps the model distinguish the two geometric formulations and gives a more accurate learned solution operator. The improvement is especially important for the rational torus, where the unconditioned model has noticeably larger error.

\section{Conclusion}
In this work, we investigated the use of Fourier neural operators for learning the geometry-dependent dynamics of the two-dimensional cubic defocusing nonlinear Schrödinger equation on rational and irrational tori. The main motivation was that the torus aspect ratio parameter \(\omega^2\) changes the dispersion relation $ \lambda_{m,\ell}=m^2+\omega^2\ell^2,$ and therefore changes the resonance and quasi-resonance structure of the Fourier lattice. As a result, the rational and irrational geometries can exhibit different rates of energy transfer to high Fourier modes. Our numerical experiments show that the learned operator captures the solution dynamics on unseen random-phase initial data and reproduces the different Sobolev-norm behavior observed on the two geometries. The ablation studies further clarify how architectural choices affect the
accuracy of the learned dynamics. The scope of the present study is shaped by the fact that the model is trained entirely on data generated by a numerical solver, so the quality of the learned operator depends on the accuracy, resolution, and time horizon of the numerical data. In addition, the experiments are performed on a fixed learning grid and over a finite prediction interval, so resolution transfer and substantially longer-time generalization are left for future work. The computational experiments were also limited by the available GPU resources and wall-time limits on the computing cluster, which affected the number of training runs, ablation configurations, grid resolutions, and long-time prediction experiments that could be included.

The results also leave open the development of resonance-aware neural operators. In the present FNO architecture, the retained Fourier modes are learned through generic spectral weights. A resonance-aware architecture would instead use the dispersion relation of the underlying PDE to identify resonant or nearly resonant Fourier-mode interactions and incorporate this information into the model architecture. By aligning the architecture more closely with the mechanism that drives spectral energy transfer, such models may improve predictive accuracy and long-time stability. The present study also leaves the more realistic waveguide setting \(\mathbb{R}\times\mathbb{T}^2\) for future investigation. The cubic NLS on \(\mathbb{R}\times\mathbb{T}^d\) has been mathematically investigated through modified scattering \cite{WilsonYu2022, HaniPausaderTzvetkovVisciglia2015}, resonant dynamics, and Sobolev-norm growth \cite{planchon2017growth}, but the corresponding geometry-aware neural-operator problem appears to remain largely unexplored.

\section*{Declarations}




\bibliographystyle{abbrv}
\bibliography{referencemain}

\end{document}